\documentclass[11pt]{article}

\def\colorful{1}

\usepackage{amsthm,amsfonts,amsmath,amssymb,epsfig,color,float,graphicx,verbatim}
\usepackage{multirow}

\newif\ifhyper\IfFileExists{hyperref.sty}{\hypertrue}{\hyperfalse}
\hypertrue
\ifhyper\usepackage{hyperref}\fi

\usepackage{enumitem}
\usepackage[capitalize]{cleveref} 

\makeatletter
\renewcommand{\section}{\@startsection{section}{1}{0pt}{-12pt}{5pt}{\large\bf}}
\renewcommand{\subsection}{\@startsection{subsection}{2}{0pt}{-12pt}{-5pt}{\normalsize\bf}}
\renewcommand{\subsubsection}{\@startsection{subsubsection}{3}{0pt}{-12pt}{-5pt}{\normalsize\bf}}
\makeatother

\usepackage{framed}

\def\nnewcolor{1}
\ifnum\nnewcolor=0

\fi
\ifnum\nnewcolor=0

\fi

\ifnum\colorful=1
\newcommand{\new}[1]{{\color{black} #1}}

\else
\newcommand{\new}[1]{{#1}}

\fi

\newtheorem{theorem}{Theorem}

\newtheorem{lemma}[theorem]{Lemma}
\newtheorem{proposition}[theorem]{Proposition}
\newtheorem{corollary}[theorem]{Corollary}
\newtheorem{claim}[theorem]{Claim}
\newtheorem{fact}[theorem]{Fact}

\newtheorem{remark}[theorem]{Remark}

\theoremstyle{definition}
\newtheorem{definition}[theorem]{Definition}

\newcommand{\R}{\mathbb{R}}
\newcommand{\Z}{\mathbb{Z}}
\newcommand{\E}{\mathbb{E}}

\newcommand{\sign}{\mathrm{sign}}
\newcommand{\opt}{\mathrm{opt}}

\newcommand{\poly}{\mathrm{poly}}

\newcommand{\dtv}{d_{\mathrm TV}}

\newcommand{\wh}[1]{{\widehat{#1}}}

\newcommand{\rnote}[1]{\footnote{{\bf [[Rocco: {#1}\bf ]] }}}
\newcommand{\inote}[1]{\footnote{{\bf [[Ilias: {#1}\bf ]] }}}

\newcommand{\ignore}[1]{}

\newcommand{\eps}{\varepsilon}

\newcommand{\norm}[1]{\left\|#1\right\|}

\newcommand{\abs}[1]{\lvert#1\rvert}
\newcommand{\len}[1]{\lvert#1\rvert}

\newcommand{\partit}{\mathcal}

\renewcommand{\eqref}[1]{Eq.~(\ref{#1})}

\newcommand{\eqdef}{\stackrel{{\mathrm {\footnotesize def}}}{=}}

\newcommand{\littlesum}{\mathop{\textstyle \sum}}

\newenvironment{algorithm}[1][\  ] %
{ \rm
\begin{tabbing}
....\=.....\=.....\=.....\=.....\=  \+ \kill
} %
{\end{tabbing} }

\floatstyle{ruled}
\newfloat{Algorithm}{H}{alg}
\newfloat{Subroutine}{H}{sub}

{
\begin{minipage}{1.0\linewidth} \begin{algorithm} %
} { \end{algorithm} \end{minipage} }

\title{Efficient Density Estimation via\\
 Piecewise Polynomial Approximation}
\ignore{
{\tiny{or}}\\
Efficient Quasi-Agnostic Learning of Piecewise-Polynomial
Probability Distributions
}

\author{Siu-On Chan\thanks{Supported by NSF award
DMS-1106999, DOD ONR grant N000141110140 and NSF award CCF-1118083.}\\
UC Berkeley \\
{\tt siuon@cs.berkeley.edu}.\\
\and
Ilias Diakonikolas\thanks{Part of this work was done while the author was at UC Berkeley supported by a Simons Postdoctoral Fellowship.}\\
University of Edinburgh\\
{\tt ilias.d@ed.ac.uk}.\\
\and
Rocco A. Servedio\thanks{Supported by NSF grants CCF-0915929 and CCF-1115703.}\\
Columbia University\\
{\tt rocco@cs.columbia.edu}.\\
\and
Xiaorui Sun\thanks{Supported by NSF grant CCF-1149257.}\\
Columbia University \\
{\tt xiaoruisun@cs.columbia.edu}.
}

\begin{document}

\maketitle

\thispagestyle{empty}

\begin{abstract}

We give a highly efficient ``semi-agnostic'' algorithm
for learning univariate probability
distributions that are well approximated by piecewise polynomial density
functions.  Let $p$ be an arbitrary distribution over an interval $I$
which is $\tau$-close (in total variation distance)
to an unknown probability distribution $q$ that is
defined by an unknown partition of $I$ into $t$ intervals and $t$
unknown degree-$d$ polynomials specifying $q$ over each of the intervals.
We give an algorithm that draws $\tilde{O}(t\new{(d+1)}/\eps^2)$ samples
from $p$, runs in time $\poly(t,d,1/\eps)$, and with high
probability outputs a piecewise polynomial hypothesis distribution $h$ that 
is $(O(\tau)+\eps)$-close (in total variation distance) to $p$.
\new{This sample complexity is essentially optimal; we show that even
for $\tau=0$, any 
algorithm that learns an unknown $t$-piecewise degree-$d$
probability distribution over $I$ to accuracy $\eps$ must 
use $\Omega({\frac {t(d+1)} {\poly(1 + \log(d+1))}} \cdot {\frac 1 {\eps^2}})$ samples
from the distribution, regardless of its running time.}
Our algorithm combines tools from approximation theory, uniform convergence,
linear programming, and dynamic programming.

We apply this general algorithm to obtain a wide range of 
results for many natural problems in density estimation
over both continuous and discrete domains.
These include state-of-the-art results for learning 
mixtures of log-concave distributions; mixtures of $t$-modal
distributions; mixtures of Monotone Hazard Rate distributions;
mixtures of Poisson Binomial Distributions; mixtures of
Gaussians; and mixtures of $k$-monotone densities.  
Our general technique
yields computationally efficient algorithms for all these problems,
in many cases with provably optimal sample complexities
(up to logarithmic factors) in all parameters.

\end{abstract}

\thispagestyle{empty}
\setcounter{page}{0}

\newpage

\section{Introduction}  \label{sec:intro}

Over the past several decades, many works in computational learning
theory have addressed the general problem of learning an unknown
Boolean function from labeled examples.  A recurring theme that has
emerged from this line of work is that state-of-the-art learning results
can often be achieved by analyzing \emph{polynomials} 
that compute or approximate
the function to be learned, see e.g.
\cite{LMN:93,KushilevitzMansour:93,Jackson:97,KlivansServedio:04jcss,
MOS:04,KOS:04}.

In the current paper we show that this theme extends to the well-studied 
unsupervised learning problem of \emph{density estimation}; namely,
learning an unknown \emph{probability distribution} given i.i.d.
samples drawn from the distribution.
We propose a new approach to density estimation based on
establishing the existence of \emph{piecewise polynomial density functions}
that approximate the distributions to be learned.
The key tool that enables this approach is a new and 
highly efficient general algorithm that we provide
for learning univariate probability 
distributions that are well approximated by piecewise polynomial density
functions. Combining our general algorithm with structural results
showing that probability distributions of interest can be well approximated
using piecewise polynomial density functions, we obtain learning
algorithms for those distributions.

We demonstrate the efficacy of this approach by showing that for many 
natural and well-studied types of distributions, there do indeed exist
piecewise polynomial densities that approximate the distributions to
high accuracy.
For all of these types of distributions our general approach gives a 
state-of-the-art computationally efficient learning algorithm 
with the best known sample complexity (number of samples that are required
		from the distribution) to date; in many cases the
sample complexity of our approach is provably optimal, 
       up to logarithmic factors in the optimal sample complexity.

       \subsection{Related work.}

       Density estimation is a well-studied topic in probability theory and
       statistics (see \cite{DG85, Silverman:86,Scott:92,DL:01} for book-length
		       introductions).  
There is a number of generic techniques for density estimation in the mathematical statistics literature,
including histograms, kernels (and variants thereof), nearest neighbor estimators, 
orthogonal series estimators, maximum likelihood (and variants thereof) and others 
(see Chapter 2 of~\cite{Silverman:86} for a survey of existing methods).
In recent years, theoretical computer science researchers have also
studied density estimation problems, with an explicit focus on obtaining
\emph{computationally efficient} algorithms (see e.g.
\cite{KMR+:94,FreundMansour:99,FOS:05focs,BelkinSinha:10, KMV:10,
MoitraValiant:10,DDS:12kmodallearn,DDS12stoc}.

We work in a PAC-type model similar to that of \cite{KMR+:94} 
and to well-studied statistical frameworks for density estimation.
The learning 
algorithm has access to i.i.d. draws from an unknown probability distribution
$p$.\ignore{ which is assumed to belong to a (known) class $\mathfrak{C}$ 
of possible
target distributions.\inote{Well, we do not really need to assume this, 
since our results are quasi-agnostic.}  \rnote{We explain the agnostic view
of our algorithms later -- I think it's simpler to start off with this
as the initial explanation.}} 
It must output a hypothesis
distribution $h$ 
such that with high probability the
total variation distance $\dtv(p,h)$
between $p$ and $h$ is at most $\eps.$
(Recall that the total variation distance between two distributions
$p$ and $h$ is ${\frac 1 2} \int |p(x)-h(x)| dx$ for 
continuous distributions, and is ${\frac 1 2} \sum |p(x)-h(x)|$
for discrete distributions.)
We shall be centrally concerned with obtaining learning algorithms that both
use few samples and are computationally efficient.

The previous work that is most closely related to our current paper is the
recent work \cite{CDSS13soda}.
(That paper dealt with distributions over the discrete domain 
${[n]} = \{1, \dots,n\}$, but since the current work focuses mostly 
on the continuous domain,
\ignore{
\inote{I think we should just use a generic interval for the intro, 
not $[-1,1]$. Same thing in all the occurrences below.}
\rnote{I think it's best to talk about a finite interval as the domain,
otherwise people may get confused thinking about how a polynomial
can approximate a distribution on a (semi)infinite interval.  If we are
going to talk about some finite interval as the domain it might as well
be $[-1,1]$, no?}}in 
our description of the \cite{CDSS13soda} results 
below we translate them to the continuous domain.  This translation
is straightforward.) 
To describe the \ main result of \cite{CDSS13soda} we need to introduce
the notions of \emph{mixture distributions} and \emph{piecewise constant}
distributions.
Given distributions $p_1,\dots,p_k$ and non-negative
values $\mu_1,\dots,\mu_k$ that sum to 1, we say that
$p = \sum_{i=1}^k \mu_i p_i$ is a \emph{$k$-mixture} of \emph{components}
$p_1,\dots,p_k$
with \emph{mixing weights} $\mu_1,\dots,\mu_k$. A draw from $p$ is obtained
by choosing $i \in [k]$ with probability $\mu_i$ and then making a draw
from $p_i$.
A distribution $q$ over {an interval $I$} is
\emph{$(\eps,t)$-piecewise constant}\ignore{(see Section~\ref{sec:prelims})}
if there is a partition of {$I$} into $t$ disjoint
intervals $I_1,\dots,I_t$
such that $p$ is $\eps$-close (in total variation distance)
to a distribution $q$ such that $q(x)=c_j$ for all $x \in I_j$ for
some $c_j \geq 0$.

The main result of \cite{CDSS13soda} is an efficient algorithm for learning
any $k$-mixture of $(\eps,t)$-piecewise constant distributions:

\begin{theorem}
\label{thm:general-cdss12}
There is an algorithm that learns any $k$-mixture
of $(\eps,t)$-piecewise constant
distributions over {an interval $I$}  to accuracy $O(\eps)$, using
$O(kt/\eps^3)$ samples and running in $\tilde{O}(k t /\eps^3)$
time.\footnote{Here and throughout the paper we work in a standard
unit-cost model
of computation, in which a sample from distribution $p$ is obtained in one
time step (and is assumed to fit into one register) and basic arithmetic
operations are assumed to take unit time.  Our algorithms, like the
\cite{CDSS13soda} algorithm, only performs basic arithmetic
operations on ``reasonable'' inputs.}
\end{theorem}

\subsection{Our main result.}
As our main algorithmic contribution, we give a significant strengthening and
generalization of Theorem~\ref{thm:general-cdss12} above.
First, we improve the {$\eps$-dependence in the}
sample complexity of Theorem~\ref{thm:general-cdss12}
from $1/\eps^3$ to a near-optimal $\tilde{O}(1/\eps^2).$ 
\footnote{Recall the well-known fact that $\Omega(1/\eps^2)$ samples are 
required for essentially every nontrivial distribution learning problem.
In particular, any algorithm that distinguishes the uniform distribution
over $[-1,1]$ from the piecewise constant distribution with pdf
$p(x) = {\frac 1 2}(1-\eps)$ for $-1 \leq x \leq 0$, $p(x)=
{\frac 1 2}(1+\eps)$ for $0 < x \leq 1$, must use $\Omega(1/\eps^2)$ 
samples.}\ignore{
\inote{I agree with Rocco's comment above, however we can significantly strengthen it very easily.
In the generic setting of the theorem it is clear that $kt/\eps^2$ is a lower bound. This is certainly true for a single $(\eps, t)$-piecewise constant distribution.
And for a $k$-mixture we can do our standard trick of having $k$ disjoint such distributions. I do not propose wasting too much space elaborating on this,
but we can certainly state it.}}
Second, we extend Theorem~\ref{thm:general-cdss12} from piecewise
constant distributions to \emph{piecewise polynomial}
distributions.  More precisely, we say that a distribution over
{an interval $I$}  is \emph{$(\eps,t)$-piecewise degree-$d$} 
if there is a partition
of {$I$} into $t$ disjoint intervals $I_1,\dots,I_t$ such that
$p$ is $\eps$-close (in total variation distance) to a distribution $q$
such that $q(x)=q_j(x)$ for all $x \in I_j$, where each of $q_1,
\dots,q_t$ is a univariate degree-$d$ polynomial.\footnote{Here
and throughout the paper, whenever we refer to a ``degree-$d$ polynomial,'' 
we mean a polynomial of degree at most $d.$}
(Note that being $(\eps,t)$-piecewise constant is the same as 
being $(\eps,t)$-piecewise degree-0.)
We say that such a distribution $q$ is a \emph{$t$-piecewise degree-$d$
distribution.}

Our main algorithmic result is the following
(see Theorem~\ref{thm:main-detail} for a fully detailed statement
of the result):

\begin{theorem} \label{thm:main2} [Informal statement]
There is an algorithm that learns any $k$-mixture
of $(\eps,t)$-piecewise degree-$d$ 
distributions over {an interval $I$} 
to accuracy $O(\eps)$, using
$\tilde{O}((d+1)kt/\eps^2)$ samples and running in $\poly((d+1),k,t,1/\eps)$
time.
\end{theorem}

As we describe below, the applications that we give for Theorem~\ref{thm:main2}
crucially use both aspects in which it strengthens 
Theorem~\ref{thm:general-cdss12} (degree $d$ rather than
degree 0, and $\tilde{O}(1/\eps^2)$ samples
rather than $O(1/\eps^3)$)
to obtain near-optimal sample complexities.

A different view on our main result, which may also be illuminating,
is that it gives a ``semi-agnostic''
algorithm for learning piecewise polynomial densities.  (Since any $k$-mixture
of $t$-piecewise degree-$d$ distributions is easily seen to be
a $kt$-piecewise degree-$d$ distribution, we phrase the discussion below
only in terms of $t$-piecewise degree-$d$ distributions rather than
mixtures.)  
{Let $\mathcal{P}_{t, d}(I)$ denote the class of all $t$-piecewise
degree-$d$ distributions over interval $I$.}
Let $p$ be any distribution over $I$.  Our 
algorithm, given parameters $t,d,\eps$ and 
\new{
$\tilde{O}(t(d+1)/\eps^2)$
} samples from $p$, outputs
an $O(t)$-piecewise degree-$d$ hypothesis distribution $h$ such that
$\dtv(p,h) \leq {4}\opt_{t, d}{(1+\eps)}+ \eps$, where
\[
\opt_{t,d}:= \inf_{r \in \mathcal{P}_{t,d}(I)} \dtv(p,r).
\]
(See Theorem~\ref{thm:no-wb}.)

\new{
We prove the following lower bound
(see Theorem~\ref{thm:lower-bound-precise} for a precise statement),
which shows that 
the number of samples that our algorithm uses is optimal up to
logarithmic factors:

\begin{theorem} \label{thm:lower-bound-informal} [Informal statement]
Any algorithm that learns an unknown $t$-piecewise degree-$d$
distribution $q$ over an interval $I$ to accuracy $\eps$ must use
$\Omega({\frac {t(d+1)} {\poly(1+\log(d+1))}} \cdot {\frac 1 {\eps^2}})$ 
samples.
\end{theorem}

Note that the lower bound holds even when the unknown distribution
is exactly a $t$-piecewise degree-$d$ distribution, i.e. $\opt_{t,d}=0$
 (in fact, the lower
bound still applies even if the $t-1$ ``breakpoints'' defining the $t$
interval boundaries within $I$ are fixed to be evenly spaced across 
$I$).
}

{
\subsection{Applications of Theorem~\ref{thm:main2}.}
Using Theorem~\ref{thm:main2} we obtain highly efficient algorithms
for a wide range of specific distribution
learning problems over both continuous and discrete domains.
These include learning mixtures of log-concave distributions;
mixtures of $t$-modal distributions; 
mixtures of Monotone Hazard Rate distributions;
mixtures of Poisson Binomial Distributions; mixtures of
Gaussians; and mixtures of $k$-monotone densities. 
(See Table~1 for a concise summary of these results
and a comparison with previous results.)
All of our algorithms run in polynomial time in all of the
relevant parameters, and for all of the mixture learning problems
listed in Table~1, our results improve on previous state-of-the-art
results by a polynomial factor.  (In some cases, such as $t$-piecewise 
degree-$d$ polynomial distributions and mixtures of $t$ bounded $k$-monotone
distributions, we believe that we give the first nontrivial
learning results for the distribution classes in question.)
In many cases the sample complexities
of our algorithms are provably optimal, up to logarithmic
factors in the optimal sample complexity.
Detailed descriptions of all of the classes 
of distributions in the table, and of our results for learning
those distributions, are given in Section~\ref{sec:applic}.

We note that all the learning results indicated with theorem
numbers in Table~1 (i.e. results proved in this paper)
are in fact \emph{semi-agnostic} learning results 
for the given classes as described in the previous
subsection; hence all of these results are 
highly robust even if the target 
distribution does not exactly belong to the specified class of distributions.
More precisely, if the target distribution is $\tau$-close
to some member of the specified class of distributions, 
then the algorithm uses the
stated number of samples and outputs a hypothesis that is
$(O(\tau)+\eps)$ close to the target distribution.

\begin{table*}[t]
\begin{center}
\begin{tabular}{|c|c|c|c|}%
\hline \bf Class of Distributions & \bf Number of samples & \bf
Reference
\\\hline\hline

\multicolumn{3}{|c|}{\bf Continuous distributions over an interval $I$}\\\hline

$t$-piecewise constant & $O(t/\epsilon^3)$ &
\cite{CDSS13soda} \\ \hline

$t$-piecewise constant & $\tilde{O}(t/\epsilon^2) \ (\dagger)$ &
Theorem~\ref{thm:main-detail}\\ \hline

$t$-piecewise degree-$d$ polynomial & $\tilde O(td/\epsilon^2) \ (\dagger)$ & 
Theorem~\ref{thm:main-detail}, Theorem~\ref{thm:lower-bound-precise}
\\\hline

log-concave & $O(1/\eps^{5/2}) \ (\dagger)$ & folklore \cite{DL:01}\\ \hline

mixture of $k$ log-concave distributions & $\tilde O(k/\eps^{5/2}) \
(\dagger)$ & Theorem~\ref{thm:lc}\\ \hline

mixture of $t$ bounded 
1-monotone distributions & $\tilde O(t/\epsilon^{3}) 
\ (\dagger)$ & Theorem~\ref{thm:kmon} \\\hline

mixture of $t$ bounded 
2-monotone distributions & $\tilde O(t/\epsilon^{5/2}) 
\ (\dagger)$ & Theorem~\ref{thm:kmon}\\\hline

mixture of $t$ bounded 
$k$-monotone distributions & $\tilde O(tk/\epsilon^{2+1/k})$ 
& Theorem~\ref{thm:kmon}\\\hline

mixture of $k$ Gaussians & $\tilde O(k/\epsilon^{2}) \ (\dagger)
$ & Corollary~\ref{cor:mixGauss} \\\hline

\hline \hline

\multicolumn{3}{|c|}{\bf Discrete distributions over $\{1,2,\dots,N\}$} 
\\\hline

$t$-modal & $\tilde O(t \log(N)/\epsilon^3) + \tilde{O}(t^3/\eps^3)$ & 
\cite{DDS:12kmodallearn} \\\hline

mixture of $k$ $t$-modal distributions & $O(kt \log(N)/\epsilon^4)$ & 
\cite{CDSS13soda} \\\hline

mixture of $k$ $t$-modal distributions & $\tilde{O}(kt \log(N)/\epsilon^3)
\ (\dagger)$ & Theorem~\ref{thm:mix-tmodal}\\\hline

mixture of $k$ monotone hazard rate distributions & $\tilde{O}(k 
\log(N)/\epsilon^4) $ & \cite{CDSS13soda}\\\hline

mixture of $k$ monotone hazard rate distributions & $\tilde{O}(k 
\log(N)/\epsilon^3) \ (\dagger)$ & Theorem~\ref{thm:MHR} \\\hline

mixture of $k$ log-concave distributions & $\tilde{O}(k /\epsilon^4) $ 
& \cite{CDSS13soda} \\\hline

mixture of $k$ log-concave distributions & $\tilde{O}(k /\epsilon^3) $ 
& Theorem~\ref{thm:logconcave}\\\hline

Poisson Binomial Distribution & $\tilde{O}(1 /\epsilon^3) $ 
& \cite{DDS12stoc,CDSS13soda}\\\hline

mixture of $k$ Poisson Binomial Distributions & $\tilde{O}(k /\epsilon^4) $ 
& \cite{CDSS13soda}\\\hline

mixture of $k$ Poisson Binomial Distributions & $\tilde{O}(k /\epsilon^3) $ 
& Theorem~\ref{thm:logconcave} \\\hline

\end{tabular}
\end{center}
\caption{Known algorithmic results for learning various 
classes of probability distributions.  
``Number of samples'' indicates the number of samples
that the algorithm uses to learn to total variation distance $\eps$.
Results given in this paper are indicated with a reference to the corresponding
theorem.
A $(\dagger)$ indicates that the given upper bound on
sample complexity is known to be optimal up to at most
logarithmic factors (i.e. 
``$\tilde{O}(m) \ (\dagger)$'' means that there is a known
lower bound of ${\Omega}(m)$).
}
\label{tab:results}
\end{table*}

\smallskip

}

\subsection{Our Approach and Techniques.} 
{
As stated in~\cite{Silverman:86}, ``the oldest and most widely used density estimator is the histogram'':
Given samples from a density $f$, the method partitions the domain into a number of intervals (bins) $I_1, \ldots, I_k$,
and outputs the empirical density which is constant within each bin. Note that the number $k$ of bins and the 
width of each bin are parameters and may depend on the particular class of distributions being learned.
Our proposed technique may naturally be viewed as a very broad generalization
of the histogram method, where instead of approximating the distribution by a {\em constant} within each bin, 
we approximate it by a {\em  low-degree polynomial.} We believe that such a generalization is very natural;
the recent paper~\cite{PA13cgs} also proposes using splines for density estimation.
(However, this is not the main focus of the paper and indeed~\cite{PA13cgs} does not  provide or analyze algorithms for density estimation.)
Our generalization of the histogram method seems likely to be of wide applicability.
Indeed, as we show in this paper, it can be used to obtain many computationally efficient learners for a wide
class of concrete learning problems, yielding several new and nearly optimal results.
}

\medskip

{

\noindent {\bf The general algorithm.}
At a high level, our algorithm uses a rather subtle dynamic program 
(roughly, to discover the ``correct'' intervals in each of which the 
underlying distribution is close to a degree-$d$ polynomial) and linear 
programming (roughly, to learn a single degree-$d$ sub-distribution 
on a given interval).  We note, however, that many challenges arise in going
from this high-level intuition to a working algorithm.

Consider first the special case in which there is only a single known
interval (see Section~\ref{sec:learn-deg-d-close}).  In this special
case our problem is somewhat reminiscent of the problem of learning a
``noisy polynomial'' that was studied by Arora and Khot \cite{AK03}. We
stress, though, that our 
setting is considerably more challenging in the following
sense:  in the \cite{AK03} framework, each data point 
is a pair $(x,y)$ where $y$
is assumed to be close to the value $p(x)$ of the target polynomial
at $x$.  In our setting the input data is \emph{unlabeled} -- we only get
points $x$ drawn from a \emph{distribution} that is $\tau$-close to
some polynomial pdf.  However, we are able to leverage some ingredients
from \cite{AK03} in our context.  We carry out a careful error
analysis using probabilistic inequalities (the VC inequality and tail bounds)
and ingredients from basic approximation theory 
to show that $\tilde{O}(d/\eps^2)$ samples suffice
for our linear program to achieve an $O(\opt_{1,d}+\eps)$-accurate
hypothesis with high probability.

Additional challenges arise when we go from a single
interval to the general case of $t$-piecewise polynomial
densities (see Section~\ref{sec:learn-piecewise-deg-d}).
The ``correct'' intervals can of course only be approximated 
rather than exactly identified, introducing an additional source of error
that needs to be carefully managed.
We formulate a dynamic program that uses the algorithm from 
Section~\ref{sec:learn-deg-d-close} as a ``black box'' to achieve our most
general learning result.
}


\ignore{
I don't know if we want to say this but: 
Our inspiration comes from Birge for learning monotone functions.
He proceeds in two steps: (1) Show that any monotone distributions 
is well-approximated by a piecewise constant density
(2) learn a piecewise constant density (histogram). 
However, (1) in Birge's case is very easy because of his 
oblivious decomposition.  Ow, one has to search for the breakpoints, 
which is essentially what we do by DP.
\rnote{I'd suggest against saying this.}}

\medskip

{

\noindent {\bf The applications.}
\ignore{For our applications we can exploit the extensive literature from 
approximation theory and related fields on piecewise polynomial (spline) 
and piecewise constant approximation of functions.
} Given our general algorithm, in order to obtain efficient learning
algorithms for specific classes of distributions, it is sufficient to
establish the existence of piecewise polynomial (or piecewise constant)
approximations to the distributions that are to be learned.
In some cases such existence results were already known; for example,
Birg\'{e} \cite{Birge:87b} provides the necessary existence result that
we require for discrete $t$-modal distributions, and classical  
results in approximation theory \cite{Dudley:74, Novak:88} give the
necessary existence results for concave distributions over continuous domains.
For log-concave densities over continuous domains, 
we prove a new structural result 
on approximation by piecewise linear densities (Lemma~\ref{lem:lc-struct})
which, combined with our general algorithm, leads 
to an optimal learning algorithm for (mixtures of) such densities.  
Finally, for \emph{$k$-monotone} distributions
we are able to leverage a recent (and quite sophisticated) 
result from the approximation theory literature \cite{KonL04, KonL07} to obtain
the required approximation result.
}

\medskip

{
\noindent {\bf Structure of this paper:} In Section~2 we include some basic preliminaries. 
In Section~3 we present our main learning result and in Section~4 we describe our applications.
}

\section{Preliminaries}

Throughout the paper for simplicity we consider distributions over the
interval $[-1,1)$.  It is easy to see that the general results given in
Section~\ref{sec:main} go through for distributions over 
an arbitrary interval $I$.  (In the applications given in 
Section~\ref{sec:applic}
we explicitly discuss the different domains over which our 
distributions are defined.)

Given a value $\kappa> 0$, 
we say that a distribution $p$ over $[-1,1)$ is \emph{$\kappa$-well-behaved}
if $\sup_{x \in [-1,1)} \Pr_{x \sim p}[x] \leq \kappa$, i.e.
no individual real value is assigned more than $\kappa$ probability under $p$.
{Any probability distribution with no atoms (and hence any 
piecewise polynomial distribution)} is 
$\kappa$-well-behaved for all $\kappa>0$, but for example the distribution
which outputs the value $0.3$ with probability $1/100$ and otherwise
outputs a uniform value in $[-1,1)$ is only $\kappa$-well-behaved
for $\kappa\geq 1/100.$
{Our results apply for general distributions over
$[-1,1)$ which may have an atomic part as well as a non-atomic part.

Throughout the paper we assume that the density $p$ is measurable.
Note that throughout the paper we only ever work with 
the probabilities $\Pr_{x \sim p}[x=z]$ of single points and probabilities
$\Pr_{x \sim p}[x \in S]$ of sets $S$ that are finite unions of intervals
and single points.}

Given a function $p: I \to \R$ on an interval ${I} \subseteq [-1,1)$
and a subinterval $J \subseteq I$, we write
$p(J)$ to denote $\int_{J} p(x) dx.$
Thus if $p$ is the pdf of a probability distribution over $[-1,1)$,
the value $p(J)$ is the probability that distribution $p$ assigns to
the subinterval $J$.  {We sometimes refer to a function $p$ over an
interval (which need not necessarily integrate to 1
over the interval) as a ``subdistribution.''}

{
Given $m$ independent samples $s_1,\dots,s_m$, drawn from  a distribution
$p$ over $[-1,1)$,
the {\em empirical distribution} $\wh{p}_m$ over $[-1,1)$
is the discrete distribution supported on $\{s_1,\dots,s_m\}$
defined as follows: for all $z \in [-1,1)$, 
$\Pr_{x \sim \wh{p}_m}[x=z] = |\{j \in [m] \mid s_j=x\}| / m$.
}

\medskip

{
\noindent {\bf Optimal piecewise polynomial approximators.}  Fix
a distribution $p$ over $[-1,1)$.  We write $\opt_{t,d}$ to denote the
value
\[
\opt_{t,d} := \inf_{r \in {\cal P}_{t,d}([-1,1))}
\dtv(p,r).
\]
Standard closure arguments can be used to show that the above infimum 
is attained by some $r \in {\cal P}_{t,d}([-1,1))$; however
this is not actually required for our purposes.\ignore{
A priori there need not exist a distribution $r \in {\cal P}_{t,d}([-1,1))$
for which $\opt_{t,d}=\dtv(p,r)$, but of course}
It is straightforward to verify that any distribution 
$\tilde{r} \in {\cal P}_{t,d}([-1,1))$
such that $\dtv(p,\tilde{r})$ is at most (say) $\opt_{t,d} + \eps/100$
is sufficient for all our arguments.
\ignore{
It is clear that for any $\eps > 0$
there exists a distribution $\tilde{r} \in {\cal P}_{t,d}([-1,1))$
such that $\dtv(p,\tilde{r})$ is at most (say) $\opt_{t,d} + \eps/100$.
It is easy to verify that wherever our arguments use the distribution 
$r$ for which $\opt_{t,d}=\dtv(p,r)$, the distribution
$\tilde{r}$ could be used instead, 
and in the final $\dtv(p,h) = c \cdot \opt_{t,d} + O(\eps)$ bounds
that we achieve only 
the constant in the $O(\eps)$ would be affected.
}
}

\ignore{
For convenience we shall sometimes assume in our arguments that there does 
indeed exist a distribution $r \in {\cal P}_{t,d}([-1,1))$ achieving
$\opt_{t,d}=\dtv(p,r)$.  This is without loss of generality because 
the end goal of our arguments is always 
to construct a distribution $h$ for which 
$\dtv(p,h) = c \cdot \opt_{t,d}
+ O(\eps)$ for some constant $c$, so even if the desired $r$ does
not exist, the distribution $\tilde{r}$ could be used in its place and only
the constant in the $O(\eps)$ would be affected.}

\medskip

\noindent {\bf Refinements.}
Let ${\cal I} = \{I_1,\dots,I_s\}$ be a partition of $[-1,1)$ into $s$
disjoint intervals, and ${\cal J} = \{J_1,\dots,J_t\}$ be a partition
of $[-1,1)$ into $t$ disjoint intervals.  We say that ${\cal J}$ is a
\emph{refinement} of ${\cal I}$ if each interval in ${\cal I}$ is a union
of intervals in ${\cal J}$, i.e. for every $a \in [s]$
there is a subset $S_a \subseteq [t]$ such that $I_a = \cup_{b \in S_a} J_b$.

For ${\cal I}= \{I_i\}_{i=1}^r$ and ${\cal I}'=\{I'_i\}_{i=1}^s$ two
partitions of $[-1,1)$ into $r$ and $s$ intervals respectively, we say that
the \emph{common refinement} of ${\cal I}$ and ${\cal I}'$ is the
partition ${\cal J}$ of $[-1,1)$ into intervals obtained from ${\cal I}$
and ${\cal I}'$ in the obvious way, by taking all possible nonempty intervals
of the form $I_i \cap I'_j.$  It is clear that ${\cal J}$ is both a refinement
of ${\cal I}$ and of ${\cal I}'$ and that ${\cal J}$ contains at most
$r+s$ intervals.

\medskip

\noindent {\bf Approximation theory.}
We will need some basic notation and 
results from approximation theory.  We write $\|p\|_{\infty}$ to
denote $\sup_{x \in [-1,1)} |p(x)|.$
We recall the famous inequalities of Bernstein and Markov bounding
the derivative of univariate polynomials:
\ignore{
\rnote{Will we use Markov -- does it help our overall bounds or no?
If not then I think there's no advantage to doing things differently
from Arora/Khot just for the sake of being different.}
\inote{My opinion is to rework the results with our improved parameters; our setting is somewhat different
and this is simple-enough. We can certainly cite them, but perhaps we should not use the words Arora-Khot more than once.}
}

\begin{theorem} \label{thm:bern-mark} For any real-valued
degree-$d$ polynomial $p$ over $[-1,1)$, we have
\begin{itemize}
\item (Bernstein's Inequality) $\|p'\|_\infty \leq \|p\|_\infty \cdot d^2;$
and
\item (Markov's Inequality) $\|p'\|_\infty \leq {\frac d {\sqrt{1-x^2}}}
\cdot \|p\|_\infty$ for all $-1 \leq x \leq 1.$
\end{itemize}
\end{theorem}

\smallskip

\noindent {\bf The VC inequality.}
Given a family of subsets $\mathcal A$ over $[-1,1)$, 
define $\norm p_{\mathcal A}
= \sup_{A\in \mathcal A} |p(A)|$.
The \emph{VC dimension} of $\mathcal A$ is the maximum size of a subset
$X\subset [-1,1)$ that is shattered by $\mathcal A$ (a set $X$ is shattered by
$\mathcal A$ if for every $Y \subseteq X$,
some $A\in\mathcal A$ satisfies
$A\cap X = Y$).
If there is a shattered subset of size $s$ for all $s$ then we say
that the VC dimension of ${\cal A}$ is $\infty$.
{The well-known \emph{Vapnik-Chervonenkis (VC) inequality} says
the following:}

\begin{theorem}[VC inequality, {\cite[p.31]{DL:01}}]
\label{thm:vc-inequality}
Let $\widehat{p}_m$ be an empirical distribution of $m$ samples from $p$.
Let $\mathcal A$ be a family of subsets of VC dimension $d$.
Then
$ \E[ \norm{p - \widehat{p}_m}_{\mathcal A}] \leq O(\sqrt{d/m}) .$
\end{theorem}

\ignore{
\noindent {\bf Uniform convergence.} \rnote{Do we use this? If not, let's 
get rid of it.}
We will also use the following uniform convergence bound:

\begin{theorem}[{\cite[p17]{DL:01}}]
\label{thm:bdd-diff}
Let $\mathcal A$ be a family of subsets over $[-1,1)$, and $\widehat{p}_m$ 
be an empirical distribution of $m$ samples from $p$.
Let $X$ be the random variable $\norm{p - \widehat{p}_m}_{\mathcal A}$.
Then we have
$ \Pr[X - \E[X] > \eta] \leq e^{-2m\eta^2}.$
\end{theorem}

}

\subsection{Partitioning into intervals of approximately equal mass.}

As a basic primitive, 
we will often need to decompose a $\kappa$-well-behaved 
distribution $p$ into 
$\Theta(1/\kappa)$ intervals each of which has probability $\Theta(\kappa)$
under $p$.  The following lemma lets us achieve this using
$\tilde{O}(1/\kappa)$ samples; the simple proof is given in 
Appendix~\ref{ap:z}.

\begin{lemma} \label{lem:part-approx-unif}
Given $0 < \kappa< 1$ and access to samples from an $\kappa/64$-well-behaved
distribution $p$ over $[-1,1)$,
the procedure {\tt Approximately-Equal-Partition} uses $\tilde{O}(1/\kappa)$
samples from $p$, runs in time $\tilde{O}(1/\kappa)$, and 
with probability at least $99/100$ outputs a partition of
$[-1,1)$ into $\ell=\Theta(1/\kappa)$ intervals such that $p(I_j)
\in [{\frac 1 {2 \kappa}}, {\frac 3 \kappa}]$ for all $1 \leq j \leq \ell.$
\end{lemma}

\section{Main result:  Learning mixtures of piecewise polynomial distributions 
with near-optimal sample complexity} \label{sec:main}

In this section we present and analyze our main algorithm for learning
mixtures of $(\tau,t)$-piecewise degree-$d$ distributions over
$[-1,1)$.  

We start by giving a simple information-theoretic argument (Proposition~\ref{prop:info}, Section~\ref{sec:ineff}) showing that there 
is a (computationally inefficient) algorithm to learn any
distribution $p$ to 
accuracy $3 \opt_{t,d} + \eps$ 
using $O(t(d+1)/\eps^2)$ samples,
where {$\opt_{t,d}$ is the smallest variation distance between $p$ and
any $t$-piecewise degree-$d$ distribution.}
\new{Next, we contrast this information-theoretic positive result with
an information-theoretic lower bound (Theorem~\ref{thm:lower-bound-precise}, Section~\ref{sec:lower-bound}) showing that any algorithm, regardless
of its running time, for learning a $t$-piecewise degree-$d$
distribution to accuracy $\eps$ must use
$\Omega({\frac {t(d+1)} {\poly(1+\log(d+1))}} \cdot {\frac 1 {\eps^2}})$
samples.}
We then build up to our main result in stages by giving efficient algorithms
for successively more challenging learning problems.

In Section~\ref{sec:learn-deg-d-close} we give an efficient ``semi-agnostic''
algorithm for learning a single degree-$d$ pdf.  
More precisely, the algorithm 
draws $\tilde{O}((d+1)/\eps^2)$ samples from any well-behaved
distribution $p$, and with high probability outputs a degree-$d$ pdf $h$
such that $\dtv(p,h) \leq 3 \opt_{1,d}{(1+\eps)}+ \eps$.
This algorithm uses ingredients from approximation theory and
linear programming.
In Section~\ref{sec:learn-piecewise-deg-d} we extend the approach using
dynamic programming to obtain an efficient ``semi-agnostic''
algorithm for $t$-piecewise degree-$d$ pdfs.
The extended algorithm 
draws $\tilde{O}(t(d+1)/\eps^2)$ samples from any well-behaved
distribution $p$, and with high probability outputs a 
$(2t-1)$-piecewise degree-$d$ pdf 
$h$ such that $\dtv(p,h) \leq 3 \opt_{t,d}{(1+\eps)}+ \eps$.
In Section~\ref{sec:mix}
we extend the result to $k$-mixtures of well-behaved distributions.
Finally, in Section~\ref{sec:kill-wb}
we show how we may get rid of the ``well-behaved'' requirement,
and thereby prove Theorem~\ref{thm:main2}.

\subsection{An information-theoretic sample complexity upper bound.}
\label{sec:ineff}

\begin{proposition} \label{prop:info}
There is a (computationally inefficient) algorithm that {draws}
$O(t(d+1) /\eps^2)$ samples from any 
distribution $p$ over $[-1,1)$, and with probability $9/10$
outputs a hypothesis distribution $h$ such that $\dtv(p,h) 
\leq 3 \opt_{t,d}+ \eps$.
\end{proposition}

\begin{proof}
The main idea is to use Theorem~\ref{thm:vc-inequality}, the VC inequality.
Let $p$ be the target distribution and let $q$ be a $t$-piecewise degree-$d$
distribution such that $\dtv(p,q) = {\opt_{t, d}}.$
The algorithm draws $m = O(t(d+1) /\eps^2)$ samples from $p$; let 
$\widehat{p}_m$ be the resulting empirical distribution of these $m$ samples.

We define the family ${\cal A}$ of subsets of $[-1,1)$ to consist
of all unions of up to $2t(d+1)$ intervals.  Since $\dtv(p,q) \leq {\opt_{t, d}}$
we have that $\|p-q\|_{\cal A} \leq {\opt_{t, d}}$.  Since the
VC dimension of ${\cal A}$ is $4t(d+1)$, Theorem~\ref{thm:vc-inequality}
implies that $\E[\|p-\widehat{p}_m\|_{{\cal A}}] \leq \eps/40$, and
hence by Markov's inequality, with probability at least $19/20$
we have that $\|p-\widehat{p}_m\|_{\cal A} \leq \eps/2.$
By the triangle inequality for $\| \cdot \|_{\cal A}$-distance,
this means that $\|q-\widehat{p}_m\|_{\cal A}
\leq {\opt_{t, d}}  + \eps/2.$

The algorithm outputs a $t$-piecewise degree-$d$ distribution $h$
that minimizes $\|h-\widehat{p}_m\|_{\cal A}$.  Since $q$ is a 
$t$-piecewise degree-$d$ distribution that satisfies $\|q-\widehat{p}_m\|_{\cal A}
\leq {\opt_{t, d}} + \eps/2$, the distribution $h$ satisfies 
$\|h - \widehat{p}_m\|_{\cal A} \leq {\opt_{t, d}} + \eps/2.$  
Hence the triangle inequality gives $\|h - q \|_{\cal A}
\leq 2 {\opt_{t, d}} + \eps.$

Now since $h$ and $q$ are both $t$-piecewise degree-$d$ distributions, they
must have at most $2t(d+1)$ crossings.  (Taking the common refinement of the
intervals for $p$ and the intervals for $q$, we get at most $2t$
intervals.  Within each such interval both $h$ and $q$ are degree-$d$
polynomials, so there are at most $2t(d+1)$ crossings in total (where
the extra $+1$ comes from the endpoints of each of the $2t$ intervals).)
Consequently we have that $\dtv(h,q) = \|h-q\|_{\cal A} \leq 2 {\opt_{t, d}}
+ \eps.$
The triangle inequality for variation distance
gives that $\dtv(h,p) \leq 3 {\opt_{t, d}}  + \eps$, and the proof is complete. 
\end{proof}

{It is not hard to see that the dependence on each of the parameters $t, d, 1/\eps$ in the above upper bound
is information-theoretically optimal.}

Note that the algorithm described above is not efficient because it is
{by no means} clear how to construct a
$t$-piecewise degree-$d$ distribution $h$
that minimizes $\|h-\widehat{p}_m\|_{\cal A}$ in a computationally efficient
way. {Indeed, several approaches to solve this problem yield running times that grow exponentially
in $t, d$}. 
Starting in Section~\ref{sec:learn-deg-d-close}, we 
give an algorithm that achieves almost
the same sample complexity but runs in time $\poly(t,d,1/\eps).$
{The main idea is that minimizing $\norm\cdot_{\mathcal A}$ (which involves
infinitely many inequalities) can be approximately achieved by minimizing a
small number of inequalities (\cref{def:interval-ineq,lem:ad-dist}), and this
can be achieved with a linear program.}

\new{
\subsection{An information-theoretic sample complexity lower bound.}
\label{sec:lower-bound}

To complement the information-theoretic upper bound from the previous
subsection, in this subsection we prove an information-theoretic lower bound
showing that even if $\opt_{t,d}=0$ (i.e. the target distribution $p$
is exactly a $t$-piecewise degree-$d$ distribution), 
$\tilde{\Omega}(t(d+1)/\eps^2)$ samples are required for any algorithm to learn
to accuracy $\eps$:

\begin{theorem}
\label{thm:lower-bound-precise}
Let $p$ be an unknown $t$-piecewise degree-$d$ distribution over 
$[-1,1)$
where $t\geq 1,$ $d \geq 0$ satisfy $t+d > 1.$
\footnote{Note that $t=1$ and $d=0$ is a degenerate case
where the only possible distribution $p$ is the
uniform distribution over $[-1,1)$.}
Let $L$ be any algorithm which, given as input
$t,d,\eps$ and access to independent samples from $p$,
outputs a hypothesis distribution $h$ such that
$\E[\dtv(p,h)] \leq \eps$, where the expectation is over
the random samples drawn from $p$ and any internal randomness of 
$L$.  Then $L$ must use at
least $\Omega({\frac {t(d+1)}{(1+\log (d+1))^2}} \cdot {\frac 1 {\eps^2}})$
samples.
\end{theorem}

Theorem~\ref{thm:lower-bound-precise} is proved using a well known
lemma of Assouad \cite{Assouad:83}, together with carefully tailored
constructions of polynomial probability density functions to meet the conditions
of Assouad's lemma.
The proof of Theorem~\ref{thm:lower-bound-precise} is deferred to Appendix~\ref{ap:lower}.
}

\subsection{Semi-agnostically learning a {degree-$d$ polynomial density}
with near-optimal sample complexity.} \label{sec:learn-deg-d-close}

In this section we prove the following:

\begin{theorem} \label{thm:agno-kis1}
Let $p$ be an ${\frac \eps {64(d+1)}}$-well-behaved 
pdf over $[-1,1)$.
There is an algorithm\\ {\tt Learn-WB-Single-Poly}$(d,\eps)$
which runs in poly$(d+1,1/\eps)$
time, uses $\tilde{O}((d+1)/\eps^2)$ samples from $p$,
and with probability at least $9/10$ outputs a degree-$d$
polynomial $q$ which defines a pdf over $[-1,1) $
such that $\dtv(p,q) \leq 3 \opt_{1,d}
{(1+\eps)} + O(\eps)$.
\end{theorem}

Some preliminary definitions will be helpful:

\begin{definition}[Uniform partition]
  Let $p$ be a subdistribution on an interval $I {\subseteq} [-1,1)$.
  A partition $\partit P = \{I_1, \dots, I_\ell\}$ of $I$ is
  \emph{$(p,\eta)$-uniform} if $p(I_j) \leq \eta$ for all $1\leq j\leq \ell$.
\end{definition}

\begin{definition}
\label{def:interval-ineq}
Let $\partit P = \{[i_0, i_1), \dots, [i_{r-1}, i_r)\}$ be a partition of an
interval $I {\subseteq} [-1,1)$.
Let $p,q:I\to \R$ be two functions on $I$.
We say that $p$ and $q$ satisfy the \emph{$(\partit P,
\eta,\eps)$-inequalities over $I$} if
\[ \abs{ p([i_j,i_\ell)) - q([i_j,i_\ell)) } \leq \sqrt{\eps(\ell-j)}\cdot 
    \eta \]
for all $0\leq j < \ell \leq r$.
\end{definition}

We will also use the following notation:  For this subsection, let $I =
{[-1,1)}$ ({$I$ will denote a subinterval of $[-1,1)$ when the results
are applied in the next subsection}).
We write $\|f\|^{(I)}_{1}$ to denote $\int_{I} |f(x)| dx$,
and we write $\dtv^{(I)}(p,q)$ to denote $\|p-q\|^{(I)}_1{/2}$.
We write $\opt^{(I)}_{1,d}$ to denote the {infimum of the} 
statistical distance $\dtv^{(I)}(p,g)$ between $p$ and any degree-$d$
subdistribution $g$ on $I$ that satisfies $g(I) = p(I)$.

\ignore{
\begin{theorem} \label{thm:kis1}
Let $p$ be a pdf over $[-1,1)$ which is an unknown degree-$d$
polynomial.  There is an algorithm {\tt Learn-Degree-$d$}
which runs in poly$(d+1,1/\eps)$
time, uses $\tilde{O}((d+1)/\eps^2)$ samples from $p$,
and with probability at least $9/10$ outputs a degree-$d$
polynomial $q$ which defines a pdf such that $\dtv(p,q) \leq \eps.$
\end{theorem}
}

The key step of {\tt Learn-WB-Single-Poly} is 
Step~3 where it calls the {\tt Find-Single-Polynomial} procedure.
In this procedure
$T_i(x)$ denotes the degree-$i$ Chebychev polynomial
of the first kind.
The function {\tt Find-Single-Polynomial} should be thought of
as the CDF of a ``quasi-distribution'' $f$; we say that
$f=F'$ is a ``quasi-distribution'' 
and not a bona fide probability distribution because it is
not guaranteed to be non-negative everywhere on $[-1,1)$.  Step~2
of {\tt Find-Single-Polynomial} processes
$f$ slightly to obtain a polynomial $q$ which is an actual distribution over
$[-1,1).$

We note that while the {\tt Find-Single-Polynomial} procedure may appear to
be more general than is needed for this section, we will exploit its
full generality in the next subsection where it is used as a key subroutine
for semi-agnostically learning $t$-piecewise polynomial distributions.

\begin{framed}
\noindent {\bf Algorithm {\tt Learn-WB-Single-Poly}:}

\medskip

\noindent {\bf Input:}  parameters $d,\eps$

\noindent {\bf Output:}  with probability at least $9/10$, a degree-$d$
distribution $q$ such that $\dtv(p,q) \leq 3 \cdot \opt_{1,d} + O(\eps)$

\begin{enumerate}

\item 
Run Algorithm~{\tt Approximately-Equal-Partition} on input
parameter $\eps/{(d+1)}$ to partition $[-1,1)$ into
$z = \Theta((d+1)/\eps)$ intervals $I_0 = [i_0,i_1)$,
$\dots,$ $I_{z-1}=[i_{z-1},i_z)$, where $i_0=-1$ and $i_z=1$, such that
for each $j \in \{1,\dots,z\}$ we have $p([i_{j-1},i_j)) =
\Theta(\eps/(d+1)).$

\item Draw $m=\tilde{O}((d+1)/\eps^2)$ samples and let $\widehat{p}_m$ be the
empirical distribution defined by these samples.

\item Call {\tt Find-Single-Polynomial}($d$, $\eps$,
$\eta:=\Theta(\eps/(d+1))$, $\{I_0,\dots,I_{z-1}\}$, $\widehat{p}_m)$ and
output the hypothesis $q$ that it returns.

\end{enumerate}

\end{framed}

\begin{framed}
\noindent {{\bf Subroutine} {\tt Find-Single-Polynomial}:}

\medskip

\noindent {\bf Input:}  degree parameter $d$;
error parameter $\eps$; parameter $\eta$; $(p,\eta)$-uniform partition
$\partit P_I = \{I_1, \dots, I_{{z}}\}$ of interval $I =
\cup_{i=1}^{{z}} I_i$ into ${{z}}$ intervals {such that $\sqrt{
\eps z}\cdot \eta \leq \eps/2$}; a subdistribution $\widehat{p}_m$ on $I$
such that $\widehat{p}_m$ and $p$ satisfy the $(\partit
P,\eta,\eps)$-inequalities over $I$

\noindent \textbf{Output:} a number $\tau$ and a degree-$d$
subdistribution $q$ on $I$ such that $q(I) = \widehat{p}_m(I)$,
\[ 
  \dtv^{(I)}(p,q) \leq 3\opt^{(I)}_{1,d}{(1+\eps)} + \sqrt{\eps r
{(d+1)}}
  \cdot \eta + {\rm error},
\]
${0\leq} \tau \leq \opt^{(I)}_{1,d} {(1+\eps)}$ {and ${\rm error} =
O({(d+1)}\eta)$}.

\ignore{
{Old guarantee that we had for {\tt Learn-WB-Single-Poly} was:

\noindent {\bf Output:}  a degree-$d$
distribution $q$ such that $\dtv(p,q) \leq 3 \cdot \opt_{1,d} + O(\eps)$
and a value $\tau$ such that $\tau \leq \opt_{1,d}$

The above generalizes it so that it can also be used for
{\tt Learn-WB-Piecewise-Poly}.
}

{Synch up the following with the Section 3.3 version}
}

\begin{enumerate}

\item Let $\tau$ be the solution to the following LP:
\[
\text{minimize~}\tau~\text{subject to the following constraints:}
\]
(Below 
$F(x) = \sum_{i=0}^{d+1} c_i T_i(x)$ where $T_i(x)$ is the degree-$i$
Chebychev polynomial of the first kind, and $f(x)=F'(x) =
\sum_{i=0}^{d+1} c_i T'_i(x)$.)

\begin{enumerate}

\item \label{item:total} $F(-1)=0$ and $F(1)={\widehat{p}_m(I)}$;

\item \label{item:phat} For each $0 \leq j < k \leq z$,
\begin{equation} \label{eq:agno-phat}
  \left| \left(\widehat{p}_m([i_j,i_k)) + \littlesum_{j\leq \ell < k} w_\ell \right) -
    (F(i_k) - F(i_j)) \right| \leq \sqrt{\eps \cdot (k-j)} \cdot \eta;
\end{equation}

\ignore{
\inote{The way Siuon writes it in the following section: $\widehat{p}+w$ and $f$ satisfy the blah inequalities.}
}

\item \label{item:robust}
\begin{align}
  \sum_{0\leq \ell < {z}} w_\ell &= 0, \\
  -y_\ell \leq w_\ell &\leq y_\ell \qquad \text{for all $0\leq \ell < {z}$,} \\
  \sum_{0\leq \ell < {z}} y_\ell &\leq 2\tau{(1+\eps)};
\end{align}

\item \label{item:AK}  The constraints $|c_i| \leq \sqrt{2}$ for
  $i=0,\dots,d+1$;

\item \label{item:AK2} The constraints
\[
0 \leq F(z)  \leq 1 \quad \text{for all~} z \in J,
\]
where $J$ is a set of {$O(d+1)^6$}\ignore{\inote{We might want to 
improve this in a later pass; ok for now.}} equally spaced points across $[-1,1]$;
\ignore{
, and
\[
\widehat{p}([0,y)) - \delta \leq
\int_{-1}^z \sum_{i=0}^d a_i T_i(x) dx  \leq \widehat{p}([0,z)) + \delta
\quad \text{for all~}z \in K,
\]
where $K$ is a set of $(d+1)^2/\delta$ equally spaced points across $[-1,1).$
}

\item \label{item:nonneg-1} The constraints
\[
\sum_{i=0}^d c_i T'_i(x) \geq 0 \quad \text{for all~}x \in K,
\]
where $K$ is a set of $O((d+1)^2/\eps)$ equally spaced points across
$[-1,1)$.

\end{enumerate}

\item Define
$q(x) = { \eps f(I)/\len I + (1-\eps)f(x)}.$
Output $q$ as the hypothesis pdf.

\end{enumerate}
\end{framed}

The rest of this subsection gives the proof of Theorem~\ref{thm:agno-kis1}.
The claimed sample complexity bound is obvious (observe that Steps~1
and~2 of {\tt Learn-WB-Single-Poly} are the only steps that draw samples),
as is the claimed running time bound (the computation is dominated by
solving the $\poly(d,1/\eps)$-size LP in {\tt Find-Single-Poly}),
so it suffices to prove correctness.

Before launching into the proof we give some intuition for the linear
program.
Intuitively $F(x)$ represents the cdf of a degree-$d$ polynomial
distribution $f$ where $f=F'.$  Constraint 1(a) captures the endpoint
constraints that any cdf must obey {if it has the same total mass as 
$\widehat p_m$}.
Intuitively, constraint 1(b)(1) ensures that for each interval $[i_j,i_k)$,
the value $F(i_k)-F(i_j)$ (which we may alternately write as
$f([i_j,i_k))$) is close to the mass
$\widehat{p}_m([i_j,i_k))$ that the empirical distribution puts on the 
interval.  
Recall that by assumption
$p$ is $\opt_{1,d}$-close to some degree-$d$ polynomial $r$.
Intuitively the variable $w_\ell$ represents $\int_{[i_\ell, i_{\ell+1})}
(r-p)$ (note that these values sum to zero by
constraint 1(c)(2)), and $y_\ell$ represents the absolute value of $w_\ell$
(see constraint~1(c)(3)).
The value $\tau$, which by constraint 1(c)(4) is at least the 
sum of the $y_\ell$'s, represents a lower bound on 
$\opt_{1,d}.$
(The factor $2$ on the RHS of constraint 1(c)(4) is present because
$\| p-r \|_1 = 2\dtv(p,r)$.)
The constraints in 1(d) and 1(e) reflect the fact that
as a cdf, $F$ should be bounded between 0 and 1 (more on this below),
and the 1(f) constraints reflect the fact that the pdf $f=F'$ should be
everywhere nonnegative (again more on this below).

\medskip

We begin by showing that with high probability
{\tt Learn-WB-Single-Poly} calls {\tt
Find-Single-Polynomial} with input parameters that satisfy
{\tt Find-Single-Polynomial}'s input requirements:

\begin{enumerate}

\item [(I)] the intervals $I_0,\dots,I_{z-1}$ are $(p,\eta)$-uniform; and

\item [(II)] $\widehat{p}_m$ and $p$ satisfy the
$(\partit P,\eta,\eps)$-inequalities over $[-1,1)$.

\end{enumerate}

We further show that given that this happens, 
{\tt Find-Single-Polynomial}'s LP is feasible and has a high-quality
optimal solution.

\begin{lemma} \label{lem:feasible}
Suppose $p$ is an ${\frac \eps {64(d+1)}}$-well-behaved
pdf over $[-1,1)$.
Then with overall probability at least $37/40$ over the random draws
performed in steps 1 and 2 of {\tt Learn-WB-Single-Poly},
conditions (I) and (II) above hold;
the LP defined in step 1 of {\tt Find-Single-Polynomial}
is feasible; and the optimal solution $\tau$ is at most $\opt_{1,d}
{\cdot (1+\eps)}.$
\end{lemma}

\begin{proof}
By Lemma~\ref{lem:part-approx-unif}, 
we have that with probability at least $99/100$, 
every pair $j<k$ is such that the true probability mass
$p([i_j,i_k))$ is $\Theta((k-j)\eps/(d+1)).$
(Note that the assumption that 
$p$ is ${\frac \eps {64(d+1)}}$-well-behaved was required to apply
Lemma~\ref{lem:part-approx-unif}.)
This gives (I).
The multiplicative Chernoff bound (and a union bound)
tells us that for every pair $(j,k)$ with $1 \leq j < k \leq z$,
with probability at least $39/40$ we have
\begin{equation}
\label{eq:phat-mult-good}
\widehat{p}_m([i_j,i_k)) \in (1 \pm \tau) p([i_j,i_k)) \quad \quad
\text{for~}\tau=\sqrt{{\frac \eps {k-j}}},
\end{equation}
and hence
\begin{equation}
\label{eq:phat-add-good}
\left|
\widehat{p}_m([i_j,i_k)) -  p([i_j,i_k)) \right| \leq
{\frac12 \cdot} \sqrt{\eps (k-j)} \cdot {\frac \eps {(d+1)}},
\end{equation}
which {implies} (II).
We assume that all these events hold going forth, and show that
then the LP is feasible.

As above, let $r$ be a degree-$d$ polynomial pdf such that $\opt_{1,d}=
\dtv(p,r)$ {and $r(I) = p(I)$}.
{Let $\overline r$ be $r$ renormalized by the empirical 
mass $\widehat{p}_m$, so
$\overline r = r\cdot \widehat{p}_m(I)/p(I)$.
Similarly let $\overline p = p\cdot \widehat{p}_m(I)/p(I)$ 
be the renormalization of
$p$.}
We exhibit a feasible solution as follows:
take $F$ to be the cdf of {$\overline r$} (a degree $d$ polynomial).
Take $w$ to be $\int_{[i_\ell,i_{\ell+1})} ({\overline r-\overline p})$,
and take $y_\ell$ to be $|w_\ell|$.
Finally, take $\tau$ to be ${\frac 1 2} \sum_{0 \leq \ell < {z}} y_\ell.$

We first argue feasibility of the above solution.  
We first take care of the easy constraints:
since $F$ is the cdf of a {sub}distribution over $I$ it is clear that
constraints 1(a) and 1(e) are satisfied,
and since both $r$ and $p$ are pdfs {with the same total mass} it is clear
that constraints 1(c)(2) and 1(f) are both satisfied. 
Constraints 1(c)(3) and 1(c)(4) also hold{, because $\frac 12\sum y_\ell =
\frac 12 \norm{r-p}_1 \cdot \widehat{p}_m(I)/p(I) \leq \dtv(p,r)\cdot (1+\eps)$,
where we have used $(\partit P, \eta, \eps)$-inequalities and the
assumption $\sqrt{\eps z}\cdot \eta \leq \eps/2$ to show $\widehat{p}_m(I)
/p(I)\in
[1-\eps/2,1+\eps/2]$}.  So it remains to argue constraints 1(b) and 1(d).

{
\begin{claim}
  \label{claim:ineq-feas}
  If $\widehat{p}_m$ and $p$ satisfy $(\partit P, \eta, \eps/4)$-inequalities on
  $I\subseteq [-1,1)$, then $\widehat{p}_m + \overline r - p$ and $\overline r$
  satisfy $(\partit P, \eta, \eps)$-inequalities on $I$.
\end{claim}

\begin{proof}
For an interval $J = [i_j, i_k) \in \partit P$, the LHS of $(\partit P, \eta,
\eps)$-inequalities between $\widehat p + (\overline r-p)$ and $\overline r$ is
\[ \abs{\widehat{p}_m(J)+({\overline r}-p)(J) - {\overline r}(J)} =
  \abs{\widehat{p}_m(J) - {p}(J)} . \]
Therefore it suffices to bound $\abs{\widehat p_m(J) - p(J)}$ and
$\abs{\overline r(J) - p(J)}$.
We can bound $\abs{\widehat p_m(J) - p(J)}$ by $(\partit P, \eta,
\eps/4)$-inequalities between $\widehat{p}_m$ and $p$ in our assumption.
We also have
  \[ \abs{\overline r(J) - p(J)} \leq \frac\eps2 p(J) \]
because $\widehat{p}_m(J)/p(J)\in [1-\eps/2,1+\eps/2]$.
\end{proof}

Note that constraint 1(b) is equivalent to $\widehat{p}_m + (\overline r - p)$ 
and $\overline r$ satisfying $(\partit P, \eps/(d+1), \eps)$-inequalities, 
therefore
this constraint is satisfied by \eqref{eq:phat-add-good} and
\cref{claim:ineq-feas}.
}

To see that constraint 1(d) is satisfied we recall some of the analysis
of Arora and Khot \cite[{Section~3}]{AK03}.  This analysis shows that since
$r$ is a cdf (a function bounded between 0 and 1 on $I$) each of its
Chebychev coefficients is at most $\sqrt{2}$ in magnitude.
{Therefore $F$ is bounded between 0 and $1+\eps$, and likewise its
coefficients are bounded by $\sqrt 2(1+\eps)$}.

To conclude the proof of the lemma we need to argue that 
$\tau \leq \opt_{1,d}{\cdot (1+\eps)}$.
Since $w_\ell = \int_{[i_\ell,i_{\ell+1})} ({\overline r-\overline p})$ it
is easy to see that $2 \tau = \sum_{0 \leq \ell < {z}} y_\ell = \sum_{0
\leq \ell < {z}} |w_\ell| \leq \|{\overline p-\overline r}\|_1$, and
hence indeed $\tau \leq \dtv(p,r){\cdot \widehat{p}_m(I)/p(I)} \leq
\opt_{1,d}{\cdot (1+\eps)}$ as required.  \end{proof}

Having established that with high probability the LP is indeed feasible,
henceforth we let $\tau$ denote the optimal solution to the LP and
$F$, $f$, $w_\ell$, $c_i$, $y_\ell$ denote the values in the optimal solution.
A simple argument (see e.g. the proof of {\cite[Theorem~8]{AK03}}) gives that $\|F\|_\infty
\leq 2{(1+\eps)}$.
Given this bound on $\|F\|_\infty$, 
the Bernstein--Markov inequality implies that $\|f\|_\infty = \|F'\|_\infty
\leq O((d+1)^2)$.
Together with (\ref{item:nonneg-1}) this implies that
$f(z) \geq -\eps{/2}$ for all $z \in [-1,1).$
Consequently $q(z) \geq 0$ for all $z \in [-1,1)$,
and
\[
\int_{-1}^1 q(x) dx = \eps + (1 - \eps) \int_{-1}^1 f(x)dx = \eps +
(1-\eps)(F(1)-F(-1)) = 1. \]
So $q(x)$ is indeed a degree-$d$ pdf.  To
prove Theorem~\ref{thm:agno-kis1}
it remains to show that $\dtv(q,p) \leq 3 \opt_{1,d} + O(\eps).$

\ignore{

It may be the case that a feasible solution $f(x)=
\sum_{i=0}^d a_i x^i$ of the LP has $f(I)<0$ for some interval $I$.
However, we now show that if $f(I)$ is negative it can
only have small magnitude, and that moreover $f$ can never take
a large-magnitude negative value on $[-1,1)$.

\begin{lemma} \label{lem:f-not-too-negative}
Let $f(x)=\sum_{i=0}^d a_i x^i$ be any feasible solution of the LP
from Steps~3(a) and~3(b).  Then

\begin{itemize}

\item For any interval $I=[u,v)
\subseteq [0,1]$ it must be the case that
$f([u,v)) \geq -\eps/d.$

\item For any point $z \in [-1,1)$ it must be the case that
$f(z) \geq -\eps.$

\end{itemize}

\end{lemma}

\begin{proof}
FILL ME IN
\end{proof}

}

We sketch the argument that we shall use to bound $\dtv(p,q).$
A key step in achieving this bound is to 
bound the $\|\cdot\|_{\cal A}$ distance between $f$ and
$\widehat{p}_m + w$ where ${\cal A} = {\mathcal A_{d+1}}$ is the class of
all unions of $d+1$ intervals and $w$ is a function based on the $w_\ell$
values (see \eqref{eq:good} below).
Similar to Section~\ref{sec:ineff}
the VC theorem gives us that $\|p - \widehat{p}_m\|_{\cal A}
\leq \eps$ with probability at least $39/40$, 
so if we can bound $\|(\widehat{p}_m +w)- f\|_{\cal A} \leq
O(\eps)$ then it will not be difficult to show that
$\|r - f\|_{\cal A} \leq 2 \opt_{1,d} + O(\eps).$
Since $r$ and $f$ are both degree-$d$ polynomials we have 
$\dtv(r,f) = \|r - f\|_{\cal A}  \leq 2 \opt_{1,d} + O(\eps)$, 
so the triangle inequality (recalling that
$\dtv(p,r) = \opt_{1,d}$) gives
$\dtv(p,f) \leq 3 \opt_{1,d}+O(\eps).$
 From this point a simple argument 
(Proposition~\ref{prop:perturb}) gives that
$\dtv(p,q) \leq \dtv(p,f) + O(\eps)$, which gives the theorem.

We will use the following lemma {that translates $(\partit P, \eta,
\eps)$-inequalities into a bound on $\mathcal A_{d+1}$ distance}.

\begin{lemma} \label{lem:ad-dist}
Let $\partit P = \{I_0=[i_0, i_1), \dots, I_{z-1}=[i_{z-1}, i_z)\}$ be a
{$(p,\eta)$-uniform} partition of $I$.
Let $\widehat{p}_m$ be a {sub}distribution {on $I$} such that
{$\widehat{p}_m$ and $p$ satisfy $(\partit P, \eta, \eps)$-inequalities on
$I$.}
If $h:I\to \R$ and $\widehat{p}_m$ also satisfy the $(\partit P, \eta,
\eps)$-inequalities, then
\[ {\|\widehat{p}_m - h\|_{\mathcal A_{{d+1}}}^{(I)} 
\leq \sqrt{\eps z {(d+1)}}\cdot \eta + {\rm error},} \]
{where ${\rm error} = O({(d+1)}\eta)$}.
\end{lemma}

\begin{proof}
To analyze
$\|\widehat{p}_m - h\|_{\mathcal A_{d+1}}$,
consider any union of ${d+1}$ 
disjoint non-overlapping intervals $S = J_1 \cup
\dots \cup J_{{d+1}}$.
We will bound $\norm{ \widehat{p}_m - h }_{\mathcal A_{d+1}}$ 
by {bounding $\abs{ \widehat{p}_m(S) - h(S)}$}.

We lengthen intervals in $S$ slightly to obtain $T = J'_1 \cup \dots \cup
J'_{{d+1}}$ 
so that each $J'_j$ is a union of intervals of the form $[i_\ell,
i_{\ell+1})$.
Formally, if $J_j = [a,b)$, then $J'_j = [a',b')$, where $a' = \max_\ell \{
i_\ell \mid i_\ell \leq a\}$ and $b' = \min_\ell \{ i_\ell\mid i_\ell \geq b
\}$.
We claim that
\begin{equation} \label{eq:lengthen}
  \abs{ \widehat{p}_m(S) - h(S) } \leq O({(d+1)}\eta) + 
\abs{ \widehat{p}_m(T) - f(T) } .
\end{equation}
Indeed, consider any interval of the form $J = [i_\ell, i_{\ell+1})$ 
such that $J \cap S \neq J \cap T$.  We have
\begin{equation} \label{eq:lengthen-single}
\abs{ \widehat{p}_m(J \cap S) - \widehat{p}_m(J \cap T) } \leq
\widehat{p}_m(J) \leq {O(\eta)},
\end{equation}
where the first inequality uses nonegativity of $\widehat{p}_m$ 
and the second inequality follows from {$(\partit P, \eta,
\eps)$-inequalities (between $\widehat{p}_m$ and $p$)} and the bound
$p([i_\ell,i_{\ell + 1})) \leq \eta$.
The {$(\partit P, \eta, \eps)$-inequalities 
(between $h$ and $\widehat{p}_m$)}
implies that the inequalities in 
\eqref{eq:lengthen-single} also hold with $h$ in place of $\widehat{p}_m$.
\ignore{and now the second inequality in \eqref{eq:lengthen-single} 
follows from condition (\ref{item:phat}).}
Now \eqref{eq:lengthen} follows by 
adding \eqref{eq:lengthen-single} across all
$J = [i_\ell, i_{\ell+1})$ such that $J\cap S\neq J\cap T$
(there are at most $2{(d+1)}$ such intervals $J$), 
since each interval $J_j$ in $S$ can change at most two such
$J$'s when lengthened.

Now rewrite $T$ as a 
disjoint union of $s \leq {d+1}$ intervals
$[i_{L_1}, i_{R_1}) \cup \dots \cup [i_{L_s}, i_{R_s})$.
We have
\[ |\widehat{p}_m(T) - h(T)| \leq \sum_{j=1}^s \sqrt{R_j - L_j} \cdot \sqrt
\eps\eta \]
by {$(\partit P, \eta, \eps)$-inequalities between $\widehat{p}_m$ and $h$}.
Now observing that 
that $0 \leq L_1 \leq R_1 \cdots \leq L_s \leq R_s \leq t =
O((d+1)/\eps)$, we get that the largest possible value of $\sum_{j=1}^s
\sqrt{R_j - L_j}$ is $\sqrt{sz} \leq {\sqrt{{(d+1)}z}}$, 
so the RHS of
(\ref{eq:lengthen}) is at most $O({(d+1)}\eta) + {\sqrt{
{(d+1)}z\eps}\eta}$, as
desired.
\end{proof}

Recall from above that $F$, $f$, $w_\ell$, $c_i$, $y_\ell$, $\tau$
denote the values in the optimal solution.
We claim that 
\begin{equation}
\label{eq:good}
 \| (\widehat{p}_m+ w) - f \|_{\cal A} = O(\eps) ,
\end{equation}
where $w$ is the sub-distribution 
which is constant on each $[i_\ell, i_{\ell+1})$
and has mass $w_\ell$ there, so in particular $\| w \|_1 \leq 2\tau
\leq 2 \opt_{1,d}{(1+\eps)}$.
Indeed, this equality follows by applying \cref{lem:ad-dist} with ${h
= f-w}$.
{The lemma requires $h$ and $\widehat{p}_m$ to satisfy $(\partit P, \eta,
\eps)$-inequalities, which follows from constraint 1(b) ($(\partit P, \eta,
\eps)$-inqualities between $\widehat{p}_m+w$ and $f$) and observing that
$(\widehat{p}_m+ w) - f = \widehat{p}_m- (f - w)$.
We have also used $\eta = \Theta(\eps/{(d+1)})$ 
to bound the {\rm error} term of the lemma by $O(\eps)$.}

Next, by the triangle inequality we have
{(writing ${\cal A}$ for ${\cal A}_{d+1}$)}
\[ \| r - f \|_{\cal A} \leq \| r - (p+w) \|_{\cal A} 
+ \| (p+w) - (\widehat{p}_m+w)
\|_{\cal A } + \| (\widehat{p}_m+w) - f \|_{\cal A} . \]
The last term on the RHS has just been shown to be $O(\eps)$.
The second term equals $\| p - \widehat{p}_m\|_{\cal A}$ and is $O(\eps)$ with 
probability at least $39/40$ by the VC inequality.
The first term is bounded by
\[ \| r-(p+w)\|_{\cal A} \leq \dtv(r, p+w) = \| r-(p+w) \|_1/2 
\leq (\| r-p\|_1 + \| w\|_1)/2 \leq 2\opt_{1,d}{(1+\eps)}. \]
Altogether, we get that $\| r - f \|_{\cal A} \leq 2\opt_{1,d}
{(1+\eps)}+ O(\eps)$.

Since $r$ and $f$ are degree $d$ polynomials, $\dtv(r,f) = \| r - f \|_{\cal A}
\leq 2\opt_{1,d}{(1+\eps)} + O(\eps)$.
This implies $\dtv(p,f) \leq \dtv(p,r) + \dtv(r,f) \leq 3\opt_{1,d}
{(1+\eps)}+  O(\eps)$.
{Finally, we turn our quasidistribution $f$ which has value $\geq -\eps/2$
everywhere into a distribution $q$ (which is nonnegative), by redistributing
the mass.}
The following simple proposition {bounds the error incurred}.

\begin{proposition} \label{prop:perturb}
{Let $f$ and $p$ be any sub-quasidistribution on $I$.}
If $q = {\eps f(I)/\len I + (1- \eps)f}$, then $\norm{q - p}_1 \leq \norm{f
- p}_1 + {\eps(f(I)+p(I))}$.
\end{proposition}

\begin{proof}
  We have
  \[ q - p = {\eps(f(I)/\len I - p) + (1-\eps)(f - p)}. \]
  Therefore
  \[ \norm{ q - p }_1 \leq { \eps \norm{f(I)/|I| - p}_1 + (1-\eps) \norm{ f
    - p }_1 \leq \eps(f(I)+p(I)) + \norm{ f - p }_1 } .
\qedhere \]
\end{proof}

We have $\dtv(p,q) \leq \dtv(p,f) + O(\eps)$ by Proposition~\ref{prop:perturb},
and we are done with the proof of Theorem~\ref{thm:agno-kis1}.
\qed

\subsection{Efficiently learning $(\eps,t)$-piecewise degree-$d$
distributions.} \label{sec:learn-piecewise-deg-d}

In this section we extend the previous result to 
semi-agnostically learn $t$-piecewise degree-$d$ distributions.
We prove the following:

\begin{theorem} \label{thm:piece-poly}
Let $p$ be an ${\frac \eps {64t(d+1)}}$-well-behaved 
pdf over $[-1,1)$.
There is an algorithm\\
{\tt Learn-WB-Piecewise-Poly}$(t,d,\eps)$
which runs in poly$(t,d+1,1/\eps)$
time, uses $\tilde{O}(t(d+1)/\eps^2)$ samples from $p$,
and with probability at least $9/10$ outputs a 
$(2t-1)$-piecewise degree-$d$ distribution $q$ such that 
$\dtv(p, q) \leq 3 \opt_{t,d} {(1+\eps)} + O(\eps)$.
\end{theorem}

At a high level, {\tt Learn-WB-Piecewise-Poly$(t,d,\eps)$}
breaks down $[-1,1)$ into $t/\eps$
subintervals (denoted as the partition $\partit P' = \{I'_0,\dots,
I'_{t/\eps - 1}\}$ in subsequent discussion;
this partition is constructed in step (\ref{item:coarsening})) 
and calls the subroutine 
\texttt{Find-Single-Polynomial}$(d,\eps,{\eta,}
\{I'_\ell,\dots,I'_{j-1}\},
\widehat{p}_m)$
on blocks of consecutive intervals from $\partit P'$ (see \cref{rem:repres}).
As shown in the previous subsection,
the subroutine {\tt Find-Single-Polynomial}
returns a degree-$d$ polynomial $h$ that is close to the
optimal degree-$d$ polynomial over $I'_\ell \cup \cdots \cup I'_{j-1}$.
An exhaustive search over all ways of breaking $[-1,1)$ up into
$t$ intervals would require running time
exponential in $t$; to improve efficiency,
dynamic programming is used to combine the different $h$'s
obtained as described above and efficiently 
construct an overall high-accuracy piecewise
degree-$d$ hypothesis.
\ignore{
To reduce the running time (in particular, to achieve the polynomial dependence
on $t$), dynamic programming is applied in step (\ref{item:dynprog}).
}

\begin{remark}
\label{rem:repres}
{The subroutine \textsc{Find-Single-Polynomial} from the previous section
assumes the domain $I$ is $[-1,1)$.
The following modification extends the subroutine to arbitrary domain $I$.}

Map the interval $I = [a,b)$ to $[-1,1)$ via
\[ \phi_I(a+\lambda(b-a)) = -1+2\lambda \quad \forall \lambda\in [0,1). \]
We write $\phi = \phi_I$ when $I$ is clear from the context.
Then the transformation $f\mapsto f_\phi$, where
\[ f_\phi(x) = \frac{b-a}2 \cdot f(\phi^{-1}(x)) , \]
is a linear map taking distributions over $I$ to distributions over $[-1,1)$
(and in fact, a linear isomorphism from $L_1(I)$ to $L_1[-1,1)$.)
This transformation is also a bijection between degree-$d$ polynomials over $I$
and those over $[-1,1)$.
As a result, if we represent $f_\phi$ by
\[ f_\phi(x) = \sum_{i=0}^d c_i T_i(x) \quad \forall x\in [-1,1), \]
where $T_i:[-1,1)\to \R$ are Chebyshev polynomials of degree $i$, we get a
representation of $f:I\to \R$ via
\begin{equation}
\label{eq:repres}
f(y) = {{\frac{2}{b-a}}} \sum_{i=0}^d c_i T_i(\phi(y)) .
\end{equation}
Note that if $f$ is bounded on $I$ and $b-a\leq 2$, then the same is true for
$f_\phi$ on $[-1,1)$, and
\[ \norm{f_\phi}_\infty^{([-1,1))} \leq 
\norm f_\infty^{(I)} . \]
(The same inequality is also true with the 
RHS multiplied by $(b-a)/2 \leq 1$, but
we only need the weaker inequality above.)

{Further, since $f\mapsto f_\phi$ preserves distances between
subdistributions, the assumptions and conclusions in the subroutine remain
unchanged.}
\end{remark}

\begin{framed}
\noindent Algorithm {\tt Learn-WB-Piecewise-Poly:}

\medskip

\noindent {\bf Input:}  parameters $t,d,\eps$

\noindent {\bf Output:}  with probability at least $9/10$, a $t$-piecewise
degree-$d$ distribution $q$ such that $\dtv(p,q) \leq 3 \cdot \opt_{t,d}
{(1+\eps)} + 
O(\eps)$

\begin{enumerate}
  \item \label{item:uniform-partit} 
Run Algorithm~{\tt Approximately-Equal-Partition} on input
parameter $\eps/(t(d+1))$ to partition $[-1,1)$ into
$z = \Theta(t(d+1)/\eps)$ intervals $I_0=[i_0,i_1)$,
$\dots,$ $I_z=[i_{z-1},i_z)$, where $i_0=0$ and $i_z=1$, such that
for each $j \in \{1,\dots,t\}$ we have $p([i_{j-1},i_j)) =
\Theta(\eps/(t(d+1))).$

  \item \label{item:coarsening} Let $s = z/(d+1) = {\Theta(}t/\eps
{)}$.
    Set $i'_j = i_{(d+1)j}$ and define interval $I'_j=[i'_j,i'_{j+1})$
for $0\leq j < s$.

  \item \label{item:empirical} Draw $m = \tilde O(t(d+1)/\eps^2)$ samples to
    define an empirical distribution $\widehat{p}_m$ over $[-1,1)$.

  \item Initialize $T(i,j) = \infty$ for $i\in \{0, \dots, 
{2t-1}\}$, $j\in \{0,
    \dots, s\}$, except that $T(0,0) = 0$.

  \item \label{item:dynprog} For $i \in \{1, \dots, 2t-1\}$, $j\in \{1, \dots,
    s\}$, $\ell\in \{0, \dots, j-1\}$:
    \begin{enumerate}
      \item Call subroutine {\tt Find-Single-Polynomial}
$(d,$ $\eps,$ $\eta=\Theta(\eps/(t(d+1))),$
$\{I'_\ell,\dots,I'_{j-1}\}$,
$\widehat{p}_m)$

      \item Let $\tau$ be the solution to the LP found by
{\tt Find-Single-Polynomial} and $h$ be the degree-$d$ hypothesis
{sub-distribution} that it returns.

      \item If $T(i,j) > T(i-1, \ell) + \tau$, then
        \begin{enumerate}
          \item Update $T(i,j)$ to $T(i-1, \ell) + \tau$
          \item Store the polynomial $h$ in a table $H(i,j)$.
        \end{enumerate}
    \end{enumerate}
  \item Recover a piecewise degree-$d$ distribution $h$ from the table
    $H(\cdot, \cdot)$.
\end{enumerate}
\end{framed}

Let $\checkmark_1$ be the event that step (\ref{item:uniform-partit}) of
{Subroutine} \texttt{Find-Piecewise-Polynomial} succeeds (i.e.\ the
intervals $[i_j, i_{j+1})$ all have mass within a constant factor of
  $\eps/t(d+1)$).
In step (\ref{item:coarsening})
{of \texttt{Learn-WB-Piecewise-Poly}}, 
the algorithm effectively constructs a coarsening
$\partit P'$ of $\partit P$ by merging every $d+1$ consecutive intervals 
from $\partit P$.
These super-intervals are used in the dynamic programming in step
(\ref{item:dynprog}).
{The table entry $T(i,j)$ stores the minimum sum of errors $\tau$ (returned
  by the subroutine \textsc{Find-Single-Polynomial}) when the interval $[i'_0,
  i'_j)$ is partitioned into $i$ pieces.
The dynamic program above only computes an estimate of $\opt_{t,d}$; one can
use standard techniques to also recover a $t$-piecewise degree-$d$ polynomial
$q$ close to $p$.
}

For step (\ref{item:empirical}), let $\checkmark_2$ be the event that $p$ and
$\widehat{p}_m$ satisfy $(\partit P, \eps/(t(d+1)),
\eps{/4})$-inequalities.  In particular, when $\checkmark_2$ holds
$\widehat{p}_m(I)/p(I) \leq {\eps/2}$ for all $I\in \partit P$.
By multiplicative Chernoff and union bound (over the $m$ samples in step
(\ref{item:empirical})), event $\checkmark_2$ holds with probability at least
$19/20$.

\begin{proposition}
\label{prop:solvable}
If $\checkmark_1$ and $\checkmark_2$ hold and $p$ is $\tau$-close to some
$t$-piecewise degree-$d$ distribution, then there is a coarsening 
$\partit P^*$ of $\partit P'$ and degree-$d$ polynomials $g_i:I^*_i
\to \R$ such that $\sum_i
\dtv(p,g_i) \leq \tau+O(\eps)$.
Further, the $g_i$ functions can be chosen to satisfy constraints
{\ref{item:total}, \ref{item:AK}--\ref{item:nonneg-1} in the subroutine
\textsc{Find-Piecewise-Polynomial}}.
\end{proposition}

\begin{proof}
  Suppose $p$ is $\tau$-close to a $t$-piecewise degree-$d$ distribution.
  In other words, there exists a partition $\{J_1, \dots, J_t\}$ of $[-1,1)$ and
    degree-$d$ polynomials $h_i:J_i\to \R$ such that $\sum_{1\leq i\leq t}
    \dtv(p, h_i) \leq \tau$.

  Let $\{ [i'_0, i'_1), \dots, [i'_{s-1}, i'_s) \}$ be $\partit P'$.
  Except in degenerate cases, the coarsening $\partit P^*$ contains $2t-1$
  intervals, corresponding to the $t$ intervals on which $p$ is a polynomial
  and $t-1$ small intervals containing ``breakpoints'' between the polynomials.
  More precisely,
  if we denote by $\{\alpha_0, \dots, \alpha_j\}$ the breakpoints of $J_1,
  \dots, J_t$ (so that $J_j = [\alpha_{j-1}, \alpha_j)$), and define
    \[ J'_j := \cup\{ [\alpha_a, \alpha_b) \mid [\alpha_a, \alpha_b) \subset
        J_j \} \]
  as the maximal subinterval of $J_j$ with endpoints from $\{\alpha_j\}$, then
  $\partit P^*$ is the partition containing all the $J'_j$'s together with
  the intervals between consecutive $J'_j$'s.
  As a result, $\partit P^*$ is a partition of $[-1,1)$ into at most $2t-1$
  non-empty intervals.

  For an interval $I^*_i$ not containing any breakpoint, the corresponding
  polynomial $g_i:I^*_i\to \R$ is simply $h_i$ rescaled by the empirical mass
  on $I^*_i$, so
  \[ g_i(x) = h_i(x) \cdot \frac{\widehat{p}_m(I^*_i)}{h_i(I^*_i)} \quad \text{for
  $x\in I^*_i\neq \emptyset$}. \]
  Then $g_i$ clearly satisfies constraints {\ref{item:total}} and
  {\ref{item:nonneg-1}}.
  Constraints \ref{item:AK} and \ref{item:AK2} are also satisfied:
  $(h_i)_{\phi_i}$ is a degree-$d$ polynomial on $[-1,1)$ bounded by $1$ in
  absolute value (here $\phi_i = \phi_{I_i^*}$), and $\widehat{p}_m(I_i^*)/h_i(I_i^*)
  \leq {\eps/}2$ when $\checkmark_2$ holds.

  For an interval $I^*_i$ containing a breakpoint, we simply set $g_i$ to be
  the constant function with total mass $\widehat{p}_m(I_i^*)$ on $I_i^*$.
  As before, $g_i$ satisfies \ref{item:AK}--\ref{item:nonneg-1}.
  The contribution of such $g_i$'s (there are at most $t-1$ of them) to $\sum_i
  \dtv(p,g_i)$ is at most $(t-1)\cdot 2\eps/t = O(\eps)$, using the fact that
  $\partit P'$ is $(\widehat{p}_m,4\eps/t)$-uniform when $\checkmark_1$ and
  $\checkmark_2$ hold.
\end{proof}

When event $\checkmark_2$ holds, $p$ and $\widehat{p}_m$ satisfy the $(\partit
P_{I^*_i}, \eps/(t(d+1)), \eps/4)$-inequalities.
But this is the same as $g_i$ and $\widehat{p}_m+(g_i-p)$ satisfying the
$(\partit P_{I^*_i}, \eps/(t(d+1)), \eps)/4$-inequalities, because
$p-\widehat{p}_m = g_i - {(}\widehat{p}_m + g_i - p)$.
{
Therefore \cref{claim:ineq-feas} tells us that constraint \ref{item:phat} is
satisfied.
Constraints \ref{item:robust} are satisfied for similar reasons as 
in Section~\ref{sec:learn-deg-d-close}.
Together with \cref{prop:solvable}}, the LP in the subroutine
\texttt{Find-Single-Polynomial} will be feasible, provided the partition
$\partit P^*$ is chosen correctly in the dynamic program.

%
%

We have the following restatement of {\cref{lem:ad-dist}}, and a robust
version as a corollary {(which follows by combining \cref{lem:ad-dist} and
the proof of \cref{prop:info})}.

\begin{lemma}[{\cref{lem:ad-dist}} restated]
  Let $\partit P$ be a $(p,\eta)$-partition of $I \subseteq [-1,1)$ into
    $r$ intervals.
  Let $\widehat{p}_m$ be a subdistribution on $I$ such that $\widehat{p}_m$ and $p$ satisfy
  the $(\partit P, \eta, \eps)$-inequalities.
  If $f:I\to \R$ and $\widehat{p}_m$ also satisfy the $(\partit P, \eta,
  \eps)$-inequalities, then
  \[ \norm{\widehat{p}_m-f}_{\mathcal A_d}^{(I)} \leq \sqrt{\eps r
  {(d+1)}}\cdot \eta + {\rm error}, \]
  {where the {\rm error} is $O({(d+1)}\eta)$.}
\end{lemma}

\begin{corollary}
  \label{cor:stat-dist-bound}
  {Let $p$ be a degree-$d$ subdistribution on $I$.}
  Let $\partit P$ be a $(p,\eta)$-partition of $I \subseteq [-1,1)$ into
    $r$ intervals.
  Let $\widehat{p}_m$ be a subdistribution on $I$ such that $\widehat{p}_m$ and
  $p$ satisfy $(\partit P, \eta, \eps)$-inequalities.
  If $h:I\to \R$ and $\widehat{p}_m+ w$ also satisfy $(\partit P, \eta,
  \eps)$-inequalities, then
  \[ \dtv^{(I)}(p,h) \leq 3\tau {(1+\eps)}
+ \sqrt{\eps r{(d+1)}}\cdot \eta +
  {\rm error}, \]
  where $2\tau = \norm w_1$ {and ${\rm error} = O({(d+1)}\eta)$}.
\end{corollary}

\begin{proof}[{\bf Proof of \cref{thm:piece-poly}}]
  Since $p$ is $\tau$-close to a $t$-piecewise degree-$d$ distribution, there
  are a partition $\{J_1, \dots, J_t\}$ of $[-1,1)$ and degree-$d$ polynomials
  $g_i:J_i\to \R$ such that $\sum_{1\leq i\leq t} \tau_i \leq \tau$, where
  $\tau_i = \dtv(p, g_i)$.
  Let $\partit P^* = \{I_1^*, \dots, I_{2t-1}^*\}$ be the coarsening of
  $\partit P'$ as in the proof of \cref{prop:solvable}.

  When $\checkmark_1$ and $\checkmark_2$ hold, it follows by a simple induction
  on $i\in \{0, \dots, 2t-1\}$ that the algorithm will output a
  $(2t-1)$-piecewise degree-$d$ distribution $h$ satisfying
  \begin{equation}
    \label{eq:stat-dist-bound-ineq}
    \dtv(p,h) \leq \sum_{1\leq i\leq t} \left( 3\tau_i {(1+\eps)} +
      \sqrt{\eps r_i {(d+1)}}
\cdot \frac\eps{t(d+1)} + O\left({(d+1)}\cdot
        \frac\eps{t(d+1)}\right) \right) + O(\eps) .
  \end{equation}
  The first term comes from \cref{cor:stat-dist-bound} (with $\eta =
  O(\eps/(t{(d+1)}))$), and the second term comes from the $t-1$ intervals containing
  the breakpoints (see the proof of \cref{prop:solvable}).
  Here $r_i$ denotes the number of intervals from $\partit P$ contained in
  $I_i^*$.
  Therefore the RHS of \eqref{eq:stat-dist-bound-ineq} is at most \[ 3\tau
    {(1+\eps)} + \sum_{1\leq i\leq t} \sqrt{\eps r_i (d+1)}\cdot
    \frac\eps{t(d+1)} + O(\eps) . \]
  The second term of this expression is bounded by $\eps$ using Cauchy--Schwarz
  and the fact that $\partit P$ contains $t(d+1)/\eps$ intervals.
\end{proof}

%
%
%
%
%
%
%
%
%


\subsection{Learning $k$-mixtures of well-behaved
$(\tau,t)$-piecewise degree-$d$ distributions.} \label{sec:mix}
In this subsection we prove Theorem~\ref{thm:main2} under
the additional restriction that the target polynomial $p$ is
well-behaved:

\begin{theorem} \label{thm:learn-wb-mix}
Let $p$ be an ${\frac \eps {64 k t(d+1)}}$-well-behaved $k$-mixture
of $(\tau,t)$-piecewise degree-$d$ 
distributions over $[-1,1)$.  There is
an algorithm that runs in $\poly(k,t,d+1,1/\eps)$ time,
uses $\tilde{O}((d+1)kt/\eps^2)$ samples from $p$, 
and with probability at least $9/10$ outputs a $(2kt-1)$-piecewise degree-$d$
hypothesis $h$ such that $\dtv(p,h) \leq 3 \opt_{t,d}
{(1+\eps)} + O(\eps).$
\end{theorem}

As we shall see, the algorithm of the previous subsection in fact suffices for
this result.  The key to extending Theorem~\ref{thm:piece-poly} to yield
Theorem~\ref{thm:learn-wb-mix} is the following structural result, which says
that any $k$-mixture of $(\tau,t)$-piecewise degree-$d$ distributions
must itself be an $(\tau,kt)$-piecewise degree-$d$ distribution.

\begin{lemma} \label{lem:mix}
Let $p_1,\dots,p_k$ each be an
$(\tau,t)$-piecewise degree-$d$ distribution
over $[-1,1)$ and let $p = \sum_{j=1}^k \mu_j p_j$ be a $k$-mixture
of components $p_1,\dots,p_k.$  Then $p$ is a
$(\tau,kt)$-piecewise degree-$d$ distribution.
\end{lemma}

The simple proof is essentially the same as the proof of Lemma~3.2
of \cite{CDSS13soda} and is given in Appendix~\ref{ap:z}.

We may rephrase Theorem~\ref{thm:piece-poly} as follows:

\medskip

\noindent {\bf Alternate Phrasing of Theorem~\ref{thm:piece-poly}.}
\emph{Let $p$ be an ${\frac \eps {64 t (d+1)}}$-well-behaved 
$(\tau,t)$-piecewise degree-$d$ pdf over $[-1,1).$
Algorithm {\tt Learn-WB-Piecewise-Poly}$(t,d,\eps)$ runs in
$\poly(t,d+1,1/\eps)$ time, uses
$\tilde{O}(t(d+1)/\eps^2)$ samples from $p$, and with probability
at least $9/10$ outputs a $(2t-1)$-piecewise degree-$d$ distribution $q$
such that $\dtv(p,q) \leq 3 \tau {(1+\eps)}+ O(\eps).$
}

\medskip

Theorem~\ref{thm:learn-wb-mix} follows immediately from 
Theorem~\ref{thm:piece-poly} and 
Lemma~\ref{lem:mix}.

\subsection{Proof of Theorem~\ref{thm:main2}.} \label{sec:kill-wb}

In this subsection we show how to remove the well-behavedness assumption
from Theorem~\ref{thm:learn-wb-mix} and thus prove Theorem~\ref{thm:main2}.
More precisely we prove the following theorem which is a more
detailed version of Theorem~\ref{thm:main2}:

\begin{theorem} \label{thm:main-detail}
Let $p$ be any $k$-mixture
of $(\tau,t)$-piecewise degree-$d$ 
distributions over $[-1,1)$.  There is
an algorithm that runs in $\poly(k,t,d+1,1/\eps)$ time,
uses $\tilde{O}((d+1)kt/\eps^2)$ samples from $p$, 
and with probability at least $9/10$ outputs a $(2kt-1)$-piecewise degree-$d$
hypothesis $h$ such that $\dtv(p,h) \leq {4 \opt_{t,d} 
(1+\eps)}+ O(\eps).$
\end{theorem}

To prove Theorem~\ref{thm:main-detail} we will need the following 
simple procedure, which (approximately) outputs
all the points in $[-1,1)$ that are $\gamma$-heavy under a distribution
$p$:

\begin{framed}
\noindent {\bf Algorithm {\tt Find-Heavy}:}

\medskip

\noindent {\bf Input:}  parameter $\gamma>0$, sample access
to distribution $p$ over $[-1,1)$

\noindent {\bf Output:} With probability at least $99/100$, a set $S
\subset [-1,1)$ such that for all $x \in [-1,1)$,

\begin{enumerate}

\item if $\Pr_{x \sim p}[x] \geq 2 \gamma$ then $x \in S$;

\item if $\Pr_{x \sim p}[x] < \gamma/2$ then $x \notin S$.

\end{enumerate}

\noindent
Draw $m = \tilde{O}(1/\gamma)$ samples from $p$.
For each $x \in [-1,1)$ let $\widehat{p}(x)$ equal $1/m$
times the number of occurrences of $x$ in these $m$ draws.
Return the set $S$ which contains all $x$ such that 
$\widehat{p}(x) \geq \gamma.$

\end{framed}

It is clear that the set $S$ returned by {\tt Find-Heavy}$(\gamma)$
has $|S| \leq 1 /\gamma$.  We now prove that
{\tt Find-Heavy} performs as
claimed:

\begin{lemma} \label{lem:FH}
With probability at least $99/100$, {\tt Find-Heavy}$(\gamma)$
returns a set $S$ satisfying conditions (1) and (2) in the ``Output''
description.
\end{lemma}

We give the straightforward proof in Appendix~\ref{ap:z}.

To prove Theorem~\ref{thm:main-detail}
it suffices to prove the following result
(which is an extension of Theorem~\ref{thm:piece-poly} 
that does not require the well-behavedness condition on $p$):

\begin{theorem}
\label{thm:no-wb}
Let $p$ be a pdf over $[-1,1)$.
There is an algorithm {\tt Learn-Piecewise-Poly}$(t,d,\eps)$
which runs in poly$(t,d+1,1/\eps)$
time, uses $\tilde{O}(t(d+1)/\eps^2)$ samples from $p$,
and with probability at least $9/10$ outputs a 
$(2t-1)$-piecewise degree-$d$ distribution $q$ such that 
$\dtv(p, q) \leq {4} \opt_{t,d}{(1+\eps)} + O(\eps)$.
where $\opt_{t,d}$ is the smallest variation distance between $p$ and
any $t$-piecewise degree-$d$ distribution.
\end{theorem}

Using the arguments of Section~\ref{sec:mix},
Theorem~\ref{thm:main-detail} follows from Theorem~\ref{thm:no-wb}
exactly as Theorem~\ref{thm:learn-wb-mix} follows from
Theorem~\ref{thm:piece-poly}.

\medskip

\noindent {\bf Proof of Theorem~\ref{thm:no-wb}.}
The algorithm 
{\tt Learn-Piecewise-Poly$(t,d,1/\eps)$}
works as follows:  it first runs
{\tt Find-Heavy}$(\gamma)$ where $\gamma = 
O({\frac \eps {t(d+1)}})$ to obtain a set $S \subset [-1,1).$
It then runs\\ {\tt Learn-WB-Piecewise-Poly-$(t,d,1/\eps)$}
but using the distribution $p_{[-1,1)\setminus S}$
(i.e. $p$ conditioned on $[-1,1) \setminus S$)
in place of $p$ throughout the algorithm.  
Each time a draw from $p_{[-1,1)\setminus S}$ is required,
it simply draws repeatedly from $p$ until a point outside of $S$
is obtained.

Let $p$ be any distribution over $[-1,1).$  
Since the conclusion of the theorem is trivial if $\opt_{t,d} \geq
1/4$, we may assume that $\opt_{t,d}< 1/4.$

Consider an execution of {\tt Learn-Piecewise-Poly$(t,d,1/\eps)$}.  
We assume that
conditions (1) and (2) of {\tt Find-Heavy} indeed hold for the set $S$
that it constructs.  Let $S' \supseteq S$ be defined as
$S' = \{x \in [-1,1): \Pr_{x \sim p}[x] \geq \gamma/2\}.$
Since every $t$-piecewise degree-$d$ distribution $q$ has
$\dtv(p,q) \geq \Pr_{x \sim p}[x \in S']$ (because $p$ assigns probability
$\Pr_{x \sim p}[x \in S']$ to $S'$ whereas $q$ assigns probability 0
to this finite set of points), it must be the case that
$\Pr_{x \sim p}[x \in S] \leq \Pr_{x \sim p}[x \in S'] \leq \opt_{t,d}.$
Hence a draw from $p_{[-1,1) \setminus S}$ is indeed a valid
draw from $p_{[-1,1) \setminus S}$ except with failure probability
at most $\opt_{t,d}< 1/4.$ It follows easily from this and the sample complexity
bound of Theorem~\ref{thm:piece-poly}  that the sample complexity
of algorithm {\tt Learn-Piecewise-Poly$(t,d,1/\eps)$} is as claimed.

Verifying correctness is also straightforward.  
{Recall that $\opt_{t,d}$ denotes the infimum of $\dtv(p,q)$ 
where $q$ is
any $t$-piecewise degree-$d$ distribution.  Fix a $q$ which achieves
$\dtv(p,q)=\opt_{t,d}$; we claim that this $q$ also satisfies
$\dtv(p_{[-1,1)\setminus S},q) \leq \opt_{t,d}.$ (To see this, note that we may
write $\dtv(p,q)$ as $A + B$ where $A$ is the contribution from points in
$[-1,1)\setminus S$ and $B$ is the contribution from $S$.  Since $\Pr_{x  
\sim q}[x \in B]$ is zero it must be the case that $B = {\frac 1 2} \Pr_{x
\sim p}[S]$, where the ``${\frac 1 2}$'' is the factor relating $L_1$
norm and total variation distance. 
Now write $\dtv(p_{[-1,1)\setminus S},q)$ as $A' + B'$ where
$A'$ is the contribution from points in $[-1,1)\setminus S$ and $B$ is the
contribution from $S$.  Clearly $B'$ is now 0, and $A'$ can be at most   
$B={\frac 1 2} \Pr_{x \sim p}[S]$ larger than $A$.)
By Lemma~\ref{lem:FH} we have that $p_{[-1,1)\setminus S}$ is $O({\frac \eps
{t(d+1)}})$-well-behaved. Hence by Theorem~\ref{thm:piece-poly}, when {\tt
Learn-WB-Piecewise-Poly}$(t,d,1/\eps)$ is run on $p_{[-1,1) \setminus S}$
it succeeds with high probability to give a hypothesis $h$ such that
$\dtv(h,p_{[-1,1) \setminus S}) \leq 3\opt_{t,d}{(1+\eps)}+
O(\eps)$. 
Since $\dtv(p,p_{[-1,1) \setminus S}) \leq \opt_{t,d}$
using the triangle inequality we get that $\dtv(h,p) \leq
4 \opt_{t,d}{(1+\eps)}+ O(\eps)$, and Theorem~\ref{thm:no-wb} is
proved. \qed
}

\ignore{OLD VERSION:

By Lemma~\ref{lem:FH} we have
that $p_{[-1,1]\setminus S}$ is 
$O({\frac \eps {t(d+1)}})$-well-behaved.
Since $\dtv(p,p_{[-1,1] \setminus S}) \leq \opt_{t,d}$
and $p$ is $(\tau,t)$-piecewise degree-$d$, by the triangle
inequality the variation distance
between $p_{[-1,1] \setminus S}$ and the closest $t$-piecewise
degree-$d$ distribution must be at most $2\opt_{t,d}.$
Hence by Theorem~\ref{thm:piece-poly}, when
{\tt Learn-WB-Piecewise-Poly}$(t,d,1/\eps)$ is run on
$p_{[-1,1] \setminus S}$ it succeeds with high probability
to give a hypothesis $h$ such that
$\dtv(h,p_{[-1,1] \setminus S}) \leq 6 \opt_{t,d}{(1+\eps)}+ O(\eps)$.
Using the triangle inequality again we get that $\dtv(h,p) \leq 7 
\tau {(1+\eps)}+ O(\eps)$, and .
Theorem~\ref{thm:no-wb} is proved.
\qed

END OLD VERSION
}

\bigskip

\section{Applications} \label{sec:applic}

In this section we use Theorem~\ref{thm:main-detail} to obtain a wide range
of concrete learning results for natural and well-studied classes of 
distributions over both continuous and discrete domains.
Throughout this section we do not aim to exhaustively cover all
possible applications of Theorem~\ref{thm:main-detail}, 
but rather to give some selected applications that are indicative 
of the generality and power of our methods.

We first (Section~\ref{sec:continuous}) give a range of applications
of Theorem~\ref{thm:main-detail} to {semi-agnostically} learn various natural classes of
continuous distributions.  These include {non-parametric 
classes such as concave, log-concave, 
and $k$-monotone densities, mixtures of these densities}, 
and {parametric classes such as} mixtures of univariate Gaussians.

Next, turning to discrete distributions
we first show (Section~\ref{sec:discrete})
how the $d=0$ case of Theorem~\ref{thm:main-detail} can be easily
adapted to learn \emph{discrete} distributions that are well-approximated
by piecewise flat distributions.  Using this general result, 
we improve prior results on learning mixtures of discrete
$t$-modal distributions,
mixtures of discrete monotone hazard rate (MHR) distributions,
and mixtures of discrete log-concave distributions (including
mixtures of Poisson Binomial Distributions), in most cases giving
essentially optimal results in terms of sample complexity.
{While we have not pursued this direction in the current
paper, which focuses chiefly on continuous distributions,
we suspect that with additional work Theorem~\ref{thm:main-detail}
can be adapted to discrete domains in its full generality (of polynomials
of degree $d$ for arbitrary $d$).  We conjecture that such 
an adaptation may give essentially optimal sample complexity bounds 
for all of the classes of discrete
distributions that we discuss in this paper.}

\subsection{Applications to Distributions over Continuous Domains.}
\label{sec:continuous}
{
In this section we apply our general approach to obtain efficient 
learning algorithms for mixtures of many different types of continuous 
probability distributions.  We focus chiefly on distributions that are 
defined by various kinds of ``shape restrictions'' on the pdf. Nonparametric density 
estimation for shape restricted classes has been a subject of study in 
statistics since the 1950s (see \cite{BBBB:72} for an early book on the topic), 
and has applications to a range of areas including 
reliability theory {(see~\cite{Reb05aos} and references therein)}. 
The shape restrictions that have been studied in this area include 
monotonicity and concavity \ignore{and convexity }of
pdfs~\cite{Grenander:56, Brunk:58, PrakasaRao:69, Wegman:70, HansonP:76, 
Groeneboom:85, Birge:87, Birge:87b}. 
More recently, motivated by statistical applications 
(see e.g. Walther's recent survey~\cite{Walther09}), 
researchers in this area have considered other types of shape 
restrictions including log-concavity and $k$-monotonicity 
\cite{BW07aos, DumbgenRufibach:09, BRW:09aos, GW09sc, BW10sn, KoenkerM:10aos}.
}

As we will see, our general method provides a single unified approach that gives a 
highly-efficient algorithm (both in terms of sample complexity and 
computational complexity) for all 
the aforementioned shape restricted densities (and mixtures thereof). 
{In most cases the sample complexities
of our efficient algorithms are optimal up to log factors.
}

\subsubsection{Concave and Log-concave Densities.} \label{ssec:logconcave}

Let $I \subseteq \R $ be a (not necessarily finite) interval.
Recall that a function $g: I \to \R$ is called {\em concave} if 
for any $x, y \in I$ and $ \lambda \in [0,1]$ it holds 
$g\left( \lambda x + (1-\lambda)y \right) \ge \lambda g(x)+(1-\lambda) g(y).$
A function $h: I \to \R_+$ is called {\em log-concave} 
if $h(x) = \exp\left( g(x) \right)$, where $g:I \to \R$ is concave.

In this section we show that our general technique yields nearly-optimal 
efficient algorithms to learn (mixtures of) concave and 
(more generally) log-concave densities. (Because of the concavity of 
the $\log$ function it is easy to see that every 
positive and concave function is log-concave.) 
In particular, we show the following:

\begin{theorem} \label{thm:lc}
Let $f: I \to \R_{+}$ be any $k$-mixture of  log-concave densities, 
where $I = [a, b]$ is an arbitrary (not necessarily finite) interval.
There is an algorithm that runs in $poly(k/\eps)$ time, draws $\tilde{O}(k / \eps^{5/2})$ samples from $f$, and with probability at least
$9/10$ outputs a hypothesis distribution $h$ such that $\dtv(f, h) \le \eps$.
\end{theorem}

We note that the above sample complexity is information-theoretically 
optimal (up to logarithmic factors). In particular, it is known 
(see e.g. Chapter 15 of~\cite{DL:01})
that learning a single concave density {(recall that a concave density is necessarily log-concave)} over $[0,1]$ 
requires $\Omega(\eps^{-5/2})$ samples. This lower bound can be easily 
generalized to show that learning a $k$-mixture of log-concave distributions 
over $[0,1]$ requires $\Omega(k/\eps^{5/2})$ samples. 
As far as we know, ours is the first {computationally} efficient algorithm with 
{essentially} optimal sample complexity for this problem.

\medskip

To prove our result we proceed as follows: We show that any log-concave 
density $f: I \to \R_{+}$ has an ${(\eps,t)}$-piecewise 
linear {(degree-1)} decomposition
for $t = \tilde{O}(1/\sqrt{\eps})$. A continuous version of the argument 
in Theorem~4.1 of~\cite{CDSS13soda}
can be used to show the existence 
of an ${(\eps, t)}$-piecewise {\em constant} {(degree-0)}
decomposition with $t = \tilde{O}(1/\eps)$. Unfortunately, the 
latter bound is essentially tight, hence cannot
lead to an algorithm with sample complexity better than $\Omega(\eps^{-3}).$ 

{
Classical approximation results (see e.g.~\cite{Dudley:74, Novak:88}) provide
optimal piecewise linear decompositions of concave functions.
While these results have a dependence on the domain size of the function,
they can rather easily be adapted to establish the existence of 
$(\eps, t)$-piecewise linear decompositions for concave densities with {$t = O(1/\sqrt{\eps})$}.
However, we are not aware of prior work establishing the existence
of piecewise linear decompositions for \emph{log-concave} densities.
We give such a result by proving the following structural lemma:
}

\ignore{
the existence of the desired piecewise linear decompositions for concave 
densities.
(It should be noted however that these old results are for concave functions, note densities and there is a dependence on the domain size of the function.
One needs to use the fact that we have a density to get if for our case.) We are not aware of any result like that for log-concave. We prove it. In particular, 
we show the following structural lemma:
}

\begin{lemma} \label{lem:lc-struct}
Let $f: I \to \R_{+}$ be any log-concave density, where $I = [a, b]$ 
is an arbitrary (not necessarily finite) interval.
There exists an ${(\eps, t)}$-piecewise linear decomposition {of $f$}
for $t = \tilde{O}(1/\sqrt{\eps})$.
\end{lemma}

{We note that our proof of Lemma~\ref{lem:lc-struct} is 
significantly different from the aforementioned known arguments 
establishing the existence of piecewise linear approximations  
for concave functions. In particular, these proofs critically exploit 
concavity, namely the fact that for a concave function $f$, 
the line segment $(x, f(x))$, $(y, f(y))$ lies below the graph of the function.}
{Before giving the proof of our lemma, we note that 
the $\tilde{O}(1/\sqrt{\eps})$ bound is best possible (up to 
log factors) even for concave densities.  This can be 
verified by considering the concave density over $[0,1]$ whose graph 
is given by the upper half of a circle.
We further note that the \cite{DL:01} $\Omega(1/\eps^{5/2})$ lower bound
implies that no significant strengthening can be achieved by using our 
general results for learning piecewise degree-$d$ polynomials for $d>1$.
} 

\medskip

\noindent Theorem~\ref{thm:lc} follows as a direct corollary of 
Lemma~\ref{lem:lc-struct} and Theorem~\ref{thm:main2}.

\medskip

\noindent {\bf Proof of Lemma~\ref{lem:lc-struct}:}
We begin by recalling the following fact which is a basic property 
(in fact an alternate characterization) of log-concave densities:

\begin{fact}(\cite{An:95}, Lemma~1)  \label{fact:lc}
Let $f: \R \to \R_{+}$ be log-concave. Suppose that $\{x \mid f(x) >0 \} = (a, b).$ Then, for all $x_1, x_2 \in (a, b)$
with $x_1< x_2$ and all $\delta \ge 0$ such that $x_1+\delta, x_2+\delta \in (a, b)$ we have
\[  \frac{f(x_1+\delta)}{f(x_1)} \ge \frac{f(x_2+\delta)}{f(x_2)}.\]
\end{fact}

Let $f$ be an arbitrary log-concave density over $\R$. 
{Well known concentration bounds for log-concave densities
(see \cite{An:95}) imply that} $1-\eps$
fraction of the total probability mass lies in a {\em finite} 
interval $[a, b]$. Let $m \in [a, b]$ be a mode of $f$
so that $f$ is non-decreasing in $[a, m]$ and 
non-increasing in $[m, b]$.  (Recall the well-known
fact \cite{An:95} that every log-concave density is unimodal, so such a mode
must exist.) 
It suffices to analyze the second {portion of the density}, i.e., 
a non-increasing log-concave (sub)-distribution over $[m, b]$. 
We may further assume without loss of generality that $[m, b] = [0,1]$. 
(It will be clear that in what follows nothing changes in the calculations
as a result of this assumption -- the length of the interval is irrelevant.)

So let $f: [0,1] \to \R_{+}$ be a 
non-increasing log-concave density and let $c = f(0) = 
\max_{x \in [0,1]} f(x).$ 
It follows from elementary calculus that $f$ is continuous in its support.
We assume {without loss of generality} that $f$ is strictly decreasing in this domain.
{(It follows from Fact~\ref{fact:lc} that for any non-increasing log-concave density over $[0, 1]$ 
there exists $x_0 \in [0,1]$ such that 
$f$ is constant in $[0, x_0]$ and strictly decreasing in $[x_0, 1]$.)}

We proceed 
to construct the desired piecewise-linear approximation in two stages:
\begin{enumerate}

\item[(a)] Let $r, s \in \Z_{+}$ with 
$r=\Theta((1/\eps) \log(1/\eps))$ and 
$s = 
{
\lceil \log_{1/(1-\eps)} \frac{f(0)}{f(1)}\rceil = 
}
\lceil \log_{1-\eps} \frac{f(1)}{f(0)}\rceil$.

We divide the domain $[0,1]$ into $t' \eqdef  \min \{ r, s \}  = 
O((1/\eps) \log(1/\eps))$ intervals 
(disjoint except at the endpoints) $\mathcal{I} = \{I_i\}_{i=1}^{t'}$, where 
$I_i = [x_{i-1}, x_i]$, $i \in [t']$. The point $x_i \in [0,1]$ 
is the point that satisfies

\begin{equation} \label{eqn:geom}
f(x_i) =  \max\{ f(x_0) (1-\eps)^i ,  f(1)\}.
\end{equation}

Since the function is strictly decreasing and continuous, such a point 
exists and is unique. 
Note that the definition with the ``max'' above addresses the case 
that $s \le r$. In this case, we will have that
$x_{t'} = x_s = 1.$ If $s>r$, then we will have 
that $f(x_i) =  f(x_0) (1-\eps)^i$ for $i \in [t']$ and $x_{t'} < 1.$

We now proceed to establish a couple of useful properties of this 
decomposition.  The first property is that the length of the 
intervals $I_i$ is non-increasing as a function of $i$ for $i \in [t']$.
\begin{claim} \label{claim:ni}
For all $i \in [t'-1]$ we have that $|I_i| \ge |I_{i+1}|$.
\end{claim}
\begin{proof}
Consider two consecutive intervals $I_i = [x_{i-1}, x_i]$ 
and $I_{i+1} = [x_{i}, x_{i+1}]$, $i \in [t'-1]$.
It is easy to see that by the definition of the intervals we have that 
\[  \frac{f(x_{i+1})}{f(x_i)} \ge \frac{f(x_{i})}{f(x_{i-1})} \]
or equivalently
\[  \frac{f(x_i + |I_{i+1}| )}{f(x_i)} \ge 
\frac{f(x_{i-1} + |I_{i}|) }{f(x_{i-1})} .\]
Since $x_{i-1} < x_i$, by Fact~\ref{fact:lc} we have
\[  \frac{f(x_{i-1} + |I_{i}| )}{f(x_{i-1})} \ge 
\frac{f(x_{i} + |I_{i}|) }{f(x_{i})} .\]
Combining the above two inequalities yields 
that $f(x_i + |I_{i+1}| ) \ge f(x_{i} + |I_{i}|)$.
Since $f$ is non-increasing we conclude 
that $x_i + |I_{i+1}| \le x_{i} + |I_{i}|$ 
and the proof is complete.
\end{proof}

The second property is that the probability mass that $f$ puts in the interval $[x_{t'}, 1]$ is bounded by $\eps$.
\begin{claim} \label{claim:tail-pc}
We have that $f([x_{t'}, 1]) \le \eps$.
\end{claim}
\begin{proof}
We consider two cases. If $t' = s$, then $x_{t'} = 1$ and the desired probability is zero.

It thus suffices to analyze the case $t' = r$. In this case $x_{t'} < 1$ 
and for all $i \in [t']$ it holds
$f(x_i) =  f(x_0) (1-\eps)^i$. Note that $f(x_{t'}) = f(0) (1-\eps)^{t'} 
\le f(0) \eps/2 = c\eps/2.$
For the purposes of the analysis, suppose we decompose $[x_{t'}, 1]$ into a 
sequence of intervals $\{I_i\}_{i > t'}$, where $I_{i} = [x_{i-1}, x_i]$ and 
point $x_i$ is defined by (\ref{eqn:geom}).
That is, we have a total of $s$ intervals $I_1, \ldots, I_s$ partitioning $[0,1]$ where by Claim~\ref{claim:ni}
$|I_1| \ge |I_2| \ge \ldots \ge | I_s|.$ 
Clearly, $\littlesum_{i=1}^s f(I_i) = 1$ and since $f$ is non-increasing
\begin{equation} \label{eqn:ineq}
c (1-\eps)^i |I_i|  \le    f(x_i) |I_i|  \le f(I_i) \le  f(x_{i-1}) |I_i|  
=  c (1-\eps)^{i-1} |I_i|.
\end{equation}
Combining the above yields 
\begin{equation} \label{eqn:ub}
c \cdot \littlesum_{i=1}^s (1-\eps)^i |I_i| \le 1.
\end{equation}
We want to show that $f([x_{t'}, 1])  =   \littlesum_{i=t'+1}^s f(I_i) 
\le \eps.$ Indeed, we have
\begin{equation} \label{eq:handy}
\littlesum_{i=t'+1}^s f(I_i) \le  
\littlesum_{i=t'+1}^s c (1-\eps)^{i-1} |I_i| \le \frac{c\eps}{2(1-\eps)} 
\cdot \littlesum_{i=1}^{s-t'} (1-\eps)^i |I_{i+t'}|
\end{equation}
where the first inequality uses (\ref{eqn:ineq}) and the second uses the 
fact that $(1-\eps)^{t'} \le \eps/2$.
By Claim~\ref{claim:ni} it follows that $|I_{i+t'}| \le |I_{i}|$ 
which yields 
\[  
\littlesum_{i=t'+1}^s f(I_i) \le \frac{c\eps}{2(1-\eps)} \cdot 
\littlesum_{i=1}^{s-t'} (1-\eps)^i |I_{i}| \le  
\frac{c\eps}{2(1-\eps)} \cdot \littlesum_{i=1}^{s} (1-\eps)^i |I_{i}|  
\le \eps
\]
where the last inequality follows from (\ref{eqn:ub}) for $\eps \le 1/2$.
\end{proof}

In fact, it is now easy to show that $\mathcal{I}$ is an $(O(\eps), t')$-flat 
decomposition of $f$, but we will not make direct use of this in 
the subsequent analysis.

\item[(b)] In the second step, we group consecutive intervals of 
$\mathcal{I}$ (in increasing order of $i$) to obtain an $(O(\eps), t)$ 
piecewise linear decomposition $\mathcal{J} = \{J_{\ell}\}_{\ell = 1}^t$ 
for $f$, where $t = \tilde{O}(\eps^{-1/2}).$ 
Suppose that we have constructed the super-intervals
$J_1, \ldots, J_{\ell-1}$ and that $\cup_{s=1}^{\ell-1} J_s = 
\cup_{k=1}^{i} I_k  = [x_0, x_{i}].$ 
{If $i=t'$ then $t$ is set to $\ell-1$, and if $i \leq t'$ then}
the super-interval $J_{\ell}$ contains the intervals $I_{i+1}, \ldots, I_j$, 
where $j \in \Z_{+}$ is the maximum value {which is $\leq t'$} 
and satisfies:

\begin{enumerate}
\item[(1)] $f(x_j) \ge f(x_i) (1-\eps)^{1/\sqrt{\eps}}$, and

\item[(2)] $|I_j| \ge (1-\sqrt{\eps}) |I_{i+1}| $.
\end{enumerate}
Within each super-interval $J_{\ell} = \cup_{k=i+1}^j{I_k} = [x_i, x_j]$ 
we approximate $f$ by the linear function $\tilde{f}$ 
satisfying $\tilde{f}(x_i) = f(x_i)$ and  $\tilde{f}(x_j) = f(x_j)$.
This completes the description of the construction.


We proceed to show correctness.  {Our first claim is that it is sufficient,
in the construction described in (b) above, to take only $t = \tilde{O}(
\eps^{-1/2})$ super-intervals, because the
probability mass under $f$ that lies to the right of the rightmost of these
super-intervals is at most $\eps$:
}

\begin{claim} \label{claim:tail-pl}
{Suppose that}
$t = \Omega(\eps^{-1/2} \log(1/\eps) )$ and $J_t = [x_u, x_v]$ 
is the rightmost super-interval.
Then, $f([x_v, 1]) \le \eps$.
\end{claim}
\begin{proof}
Consider a generic {super-interval} 
$J_{\ell} = \cup_{k=i+1}^{j} I_k$. 
Since $j$ is the maximum value that satisfies both (1) and (2) we 
conclude that either 
\begin{equation} \label{eqn:cond1}
j+1 - i > 1/\sqrt{\eps}
\end{equation}
(this inequality follows from the negation of (1) and the definition of $f(x_i)$, $f(x_j)$) 
or 
\begin{equation} \label{eqn:cond2}
|I_{j+1}| < (1-\sqrt{\eps}) |I_{i+1}| .
\end{equation}

Suppose we {have} $t = \Omega ( \eps^{-1/2} \log(1/\eps) )$ 
super-intervals. Then, either (\ref{eqn:cond1}) is satisfied 
for at least $t/2$ super-intervals or  (\ref{eqn:cond2}) 
is satisfied for at least $t/2$ super-intervals. 
Denote the rightmost super-interval by $J_t = [x_u, x_v]$. 
In the first case, for an appropriate constant in the big-Omega,
we have $v = t'$ and the desired result follows from Claim~\ref{claim:tail-pc}. 

In the second case, 
{for an appropriate constant in the big-Omega}
we will have $|I_v| \le \eps^{{3}} 
|I_1|$. To show that $f([x_v, 1]) \le \eps$ 
in this case,
we consider further partitioning the interval $[x_v, 1]$ into a 
 sequence of intervals $\{I_i\}_{i > v}$, 
where $I_{i} = [x_{i-1}, x_i]$ and point $x_i$ is defined by (\ref{eqn:geom}).
By Claim~\ref{claim:ni} we will have that $|I_i| \le |I_v|$, $i>v$. 
We can therefore bound the desired quantity by
\[
{\littlesum_{i=v+1}^s f(I_i) \leq}
\littlesum_{i=v+1}^s c (1-\eps)^{i-1} |I_i| \le 
\littlesum_{i=v+1}^s c (1-\eps)^{i-1} \eps^{{3}} 
|I_{{1}}| \le \eps^{{3}} c |I_1| 
\littlesum_{i=1}^{\infty} (1-\eps)^{i-1}
\le \frac{\eps^{{3}}}{{(1-\eps)^2}} \cdot \frac{1-\eps}{\eps} 
\leq \eps,  \]
where 
{the first inequality used the first inequality of (\ref{eq:handy}) and}
the {penultimate} inequality uses the fact that $c (1-\eps) |I_1| 
\le p(I_1) \le 1.$
This completes the proof of the claim.

\end{proof}

The main claim we are going to establish for the piecewise-linear approximation $\mathcal{J}$ is the following:
\begin{claim} \label{claim:pl}
For any super-interval $J_{\ell} = \cup_{k=i+1}^{j} I_k$ and any $i \le m \le j$ we have that 
\[  | \tilde{f}(x_m) - f(x_m) | = O(\eps) f(x_m).\]
\end{claim}

Assuming the above claim it is easy to argue that $\mathcal{J}$ is indeed an $(O(\eps), t)$ piecewise linear approximation to $f$.
Let $\tilde{f}$ be the piecewise linear function over $[0,1]$ which is linear over $J_{\ell}$ (as described above) and identically zero
in the interval $[x_v, 1]$.

Indeed, we have that
\begin{eqnarray*}
\| \tilde{f} - f \|_1 &\le& \sum_{\ell=1}^t  \int_{J_{\ell}}  | \tilde{f}(y) - f(y)| dy + f([x_v, 1]) \\ 
&\le& \sum_{i=1}^v  \int_{y=x_{i-1}}^{x_i}  | \tilde{f}(y) - f(y)| dy + \eps \\
&\le& {\sum_{i=1}^v  O(\eps) f(x_{m}) |I_{i}| + \eps} \\    
&\le& \sum_{i=1}^v  O(\eps) f(x_{i-1}) |I_{i}| + \eps \\    
&=& O(\eps)                      
\end{eqnarray*}
where {the second inequality used Claim~\ref{claim:tail-pl}}, 
{the third inequality used Claim~\ref{claim:pl}},
{the fourth inequality used the fact that $f$ is non-increasing},
and the {final inequality used the fact that}
\[
\sum_{i=1}^v  f(x_{i-1}) |I_{i}| 
\le
{
{\frac 1 {1-\eps}}\sum_{i=1}^v  f(x_i) |I_{i}| 
\le
{\frac 1 {1-\eps}}\sum_{i=1}^v  f(I_{i})
}
\le 1/(1-\eps),
\]
which follows by the definition of the $f(x_i)$'s.
\ignore{\inote{This 
basically follows from (2), clean it up in later pass.}.
}

We are now ready to give the proof of the claim.

\noindent \begin{proof}[Proof of Claim~\ref{claim:pl}]
If $\tilde{f}$ is the approximating line between $x_i$ and $x_j$ we can 
write
\[ 
\tilde{f}(x_m) = f(x_i) + (f(x_j) - f(x_i)) 
\cdot \frac{\littlesum_{k=i+1}^m {|}I_k{|}  }{\littlesum_{k=i+1}^j 
{|}I_k{|} }.
\]
Note that $f(x_j) - f(x_i) = f(x_i) \left( (1-\eps)^{{j-i}}-1 \right)$. 
We also recall that 
\[  (1-\eps)^{j-i} = 1- \eps(j-i) + \eps^2 (j-i)^2/2 + O( \eps^3 (j-i)^3).\]
Since $i, j$ are in the same super-interval, we have 
that $j-i \le 1/\sqrt{\eps}$, which implies that the
above error term is $O(\eps^{3/2})$.  We will use this approximation 
henceforth, which is also valid for any $m \in [i,j]$.

Also by condition (2) defining the lengths of the intervals in the same 
super-interval
and the monotonicity of the lengths themselves, we obtain
\[ 
\frac{m-i}{j-i} \cdot (1-\sqrt{\eps}) \le  
\frac{\littlesum_{k=i+1}^m {|}I_k {|}  }{\littlesum_{k=i+1}^j 
{|}I_k {|} } 
\le \frac{m-i}{j-i} \cdot \frac{1}{1-\sqrt{\eps}}.  
\]
By carefully combining the above inequalities we obtain the desired result.
In particular, we have that 
\[  \tilde{f}(x_m) \le f(x_i) \left[  1- \eps \big(1+O(\sqrt{\eps})\big)(m-i) 
+(\eps^2/2)(j-i)(m-i)\big(1+O(\sqrt{\eps})\big) +O(\eps^{3/2})     \right]. \]
Also
\[ 
f(x_m) = f(x_i) \left[  1-\eps(m-i) +(\eps^2/2)(m-i)^2 + O(\eps^{3/2})
\right].  
\]
Therefore, {using the fact that $j-i,m-i \leq 1/\sqrt{\eps}$, we get 
that}
\[  \tilde{f}(x_m) - f(x_m) \le O(\eps) f(x_i). \]
{In an analogous manner we obtain that
\[  f(x_m) - \tilde{f}(x_m) \le O(\eps) f(x_i). \]
}
By the definition of a {super-interval}, 
the maximum and minimum values of $f$ within the {super-interval}
are within a $1+o(1)$ factor of each other.
This completes the proof of Claim~\ref{claim:pl}.
\end{proof}

\end{enumerate}
 
This completes the proof of Lemma~\ref{lem:lc-struct}.
\qed

\subsubsection{$k$-monotone Densities.}
Let $I = [a, b] \subseteq \R$ be a (not necessarily finite) interval.
A function $f: I \to \R_{+}$ is {said to be} \emph{$1$-monotone} 
if it is non-increasing.
It is \emph{$2$-monotone} if it is non-increasing and convex, and 
\emph{$k$-monotone} for $k \ge 3$ if 
$(-1)^j f^{(j)}$ is non-negative, 
non-increasing and convex for $j=0, \ldots, k-2.$
The problem of density estimation for $k$-monotone densities has been extensively investigated
in the mathematical statistics community during the past few years (see~\cite{BW07aos, GW09sc, 
BW10sn, S10ams} and references therein)
due to its significance in both theory and applications~\cite{BW07aos}. 
For example, as pointed out in~\cite{BW07aos}, the problem
of learning an unknown $k$-monotone density arises in 
a generalization of Hampel's bird-watching problem~\cite{Hampel87}.

The aforementioned papers from the statistics community 
focus on analyzing the rate of convergence 
of the Maximum Likelihood Estimator (MLE) under various metrics. 
In this section we show that our approach yields an efficient algorithm 
to learn bounded $k$-monotone densities over $[0, 1]$
{(i.e., $k$-monotone densities $p$ such that 
$\sup_{x \in [0,1]} p(x) = O(1)$)},
and mixtures thereof, with sample complexity $\tilde{O}(k/\eps^{2+1/k})$. 
This bound is provably optimal {(up to log factors)}
{for $k=1$ by \cite{Birge:87} and for $k=2$ (see e.g. Chapter 15 of \cite{DL:01}})
and we conjecture that it is {similarly} tight for all values of $k$.

\ignore{
\inote{It seems to be that finding the optimal learning algorithm is open. In fact, the sample complexity we get
is what these guys conjecture for the MLE but cannot prove. Also: I think one of the papers shows a tight bound on the cover size, from which
we get the same sample complexity upper bound non-constructively. Do we know a lower bound on sample complexity?}

\inote{For monotone we know the tight answer from Birg{\'e}~\cite{Birge:87, Birge:87b}
Let $f: [a, a+L] \to [0, H]$ be non-decreasing (or non-increasing). Then, there exists an algorithm to learn $f$ 
up to variation distance $\eps$ using $O(\log(1+HL)/\eps^3)$  samples and this bound is information-theoretically optimal (up to a constant factor).
The upper bound is shown in~\cite{Birge:87b} by showing that any such density $f$ has an $(\eps, t)$-piecewise constant approximation
with $O(\log(1+HL)/\eps)$ pieces. The important thing is that this decomposition is in fact {\em  oblivious}, i.e. it is the same for any monotone $f$.
As a consequence, there is a very simple algorithm to learn.
(For other classes of distributions such decompositions may exist but are not oblivious and a fundamental difficulty is to find/approximate such a decomposition using samples.)} 
}

Our main algorithmic result for $k$-monotone densities is the following:

\begin{theorem} \label{thm:kmon}
Let $k \in \Z_{+}$ and $f:[0,1] \to \R_{+}$ be a $t$-mixture of 
bounded $k$-monotone densities.
There is an algorithm that runs in $\poly(k, t, 1/\eps)$ time, 
uses $\tilde{O}(t{k}/\eps^{2+1/k})$ samples, 
and outputs a hypothesis distribution $h$ such that $\dtv(h, f) \le \eps$.
\end{theorem}

\noindent The above theorem follows as a corollary of Theorem~\ref{thm:main2} and the following structural result:

\begin{lemma}[{Implicit in }\cite{KonL04, KonL07}] \label{lem:kmon-struct}
Let $f:[0,1] \to \R_{+}$ be a $k$-monotone density such that 
$\sup_{x} |f(x)| = O(1)$. There exists an $(\eps, t)$-piecewise degree-$(k-1)$
approximation of $f$ with $t = O(\eps^{1/k})$.
\end{lemma}

\noindent {As we now explain} 
the above lemma can be deduced from recent work in approximation 
theory~\cite{KonL04, KonL07}. To state the relevant
theorem we need some terminology: Let $s \in \Z_+$, and for a real function 
$f$ over interval $I$, let 
$\Delta^s_{\tau}f(t) = \littlesum_{i=0}^s (-1)^{s-i}\binom{s}{i}f(t+i\tau)$ 
be the $s$th difference of the function $x$ with step $\tau>0$, where
$[t, t+s\tau] \subseteq I.$ For $r \in \Z_{+}^{\ast}$, let $W^r_1 (I)$ 
be the set of real functions $f$ over $I$ that are absolutely continuous 
in every compact subinterval of $I$ and satisfy $\|f^{(r)}\|_1 = O(1).$ 
We denote by $\Delta^s_+ W^r_1 (I)$ the subset of functions $f$ in $W^r_1(I)$ 
that satisfy $\Delta^s_{\tau}f(t) \ge 0$ for all $\tau>0$ such 
that $[t, t+s\tau] \subseteq I.$ (Note that if $f$ is $s$-times 
differentiable the latter condition is tantamount to
saying that $f^{(s)} \ge 0$.) We have the following:

\begin{theorem}[Theorem 1 in~\cite{KonL07}] \label{thm:kl}
Let $s \in \Z_{+}$, $r, \nu, n \in \Z_{+}^{\ast}$ such that $\nu \ge \max\{ r,s \}$. 
For any $f \in \Delta^s_+ W^r_1(I) $ there exists a piecewise degree-$(\nu-1)$ polynomial approximation $h$ to $f$ with $n$ pieces such that 
$\|h-f\|_1  = O(n^{-{\max\{r, s\}}}).$ 
\end{theorem}

\noindent (In fact, it is shown in~\cite{KonL07} that the above bound is 
quantitatively optimal up to constant factors.)
Let $f: [0, 1] \to \R_{+}$ be a $k$-monotone density such that 
$\sup |f| = O(1).$ It is easy to see that Lemma~\ref{lem:kmon-struct} 
follows from Theorem~\ref{thm:kl} for the following setting of 
parameters: $s=k$, $r=1$ and $\nu = \max \{r, s\} = k.$ Indeed, 
since $(-1)^{k-2} f^{(k-2)}$ is convex, it follows
that $\Delta^k_{\tau}f(t)$ is nonnegative for even $k$ and nonpositive for 
odd $k$.

Since $f$ is a non-increasing bounded density, it is clear that 
$\|f'\|_1 = |\int_{{0}}^1 f'(t)dt| =  f(0) - f(1)  = O(1).$ Hence, for even $k$ 
Theorem~\ref{thm:kl} is applicable to $f$ and yields 
Lemma~\ref{lem:kmon-struct}.
For odd $k$, Lemma~\ref{lem:kmon-struct} follows by applying Theorem~\ref{thm:kl} to the function {$-f$}.

\ignore{
\inote{ The reason I am phrasing the above theorem like that is because it is 
not clear what the right dependence on ($H = $ max value of $f$) really is. 
It would be nice to get the right dependence on the maximum value $H$. It 
should be $(\log H)^{1/k}$. At least this is the case for monotone (see 
discussion above) and agrees with the nonconstructive cover bound. If we do, 
then we can crap on kernel methods, which get linear in $H$ even for the 
monotone case (this is also mentioned in Birge's paper and the DL book).}
}

\ignore{
\inote{The above sample complexity is tight for $k=1, 2$. This follows from Birge for $k=1$ and for $k=2$ from~\cite{DL:01}.
We conjecture it is tight for all $k$.}
}

\subsubsection{{Mixtures of Univariate Gaussians.}}
\label{sec:mix-gauss}

As a final example {illustrating} the power and generality of {Theorem~\ref{thm:main2}},
we now show how it very easily yields a computationally efficient
and essentially optimal (up to logarithmic factors) sample complexity 
algorithm for learning mixtures of $k$ univariate Gaussians.
As will be evident from the proof, similar results could be obtained
via our techniques for a wide range of mixture distribution learning problems
for different types of parametric univariate distributions
beyond Gaussians.

\begin{lemma} \label{lem:mixGauss}
Let $p=N(\mu,\sigma^2)$ be a univariate Gaussian.  Then $p$ is an
$(\eps,3)$-piecewise degree-$d$ distribution for $d = O(\log(1/\eps)).$
\end{lemma}

{Since Theorem~\ref{thm:main-detail} is easily seen to extend to
{semi}-agnostic learning of $k$-mixtures of $t$-piecewise
degree-$d$ distributions,}
Lemma~\ref{lem:mixGauss} 
immediately gives the following {semi}-agnostic learning
result for mixtures of $k$ one-dimensional Gaussians:
\ignore{
\rnote{Strictly speaking Theorem~\ref{thm:main-detail}
does not quite give this.  Theorem~\ref{thm:main-detail} assumes
that the target is actually a $k$-mixture of $(\eps,t)$-piecewise
degree-$d$ distributions, so it would apply
to mixtures of Gaussians.  For the quasi-agnostic statement below,
where the target is only *close* to a mixture of Gaussians,
one would need a small bit of additional argumentation.  
I guess the way to do this would be to say that 
if the target distribution $q$ is $\eps$-close
to a $k$-mixture of $(\eps,t)$-piecewise
degree-$d$ distributions, then in fact it is
exactly a $k$-mixture of $(2\eps,t)$-piecewise
degree-$d$ distributions.  (This is true, right?)  What is the best
way to handle this expositionally:  with a note after the statement
of Theorem~\ref{thm:main-detail} explaining this, or dealing
with it here?  

I guess it's best to do it near the statement
of Theorem~\ref{thm:main-detail}, and then say there once and for all
that all the learning results we obtain with our methods are ``robust'' --
i.e. when we state a result for learning a particular class 
$\mathfrak{C}$ of distributions (like
mixtures of $t$-modal distributions or mixtures of Gaussians or whatever)
to accuracy $O(\eps)$,
actually the algorithm learns any distribution that is $\eps$-close
to the class.  We may even want to introduce a specific term 
``robust learning'' or something like that for this concept, 
and use it for all the applications.  Let me know what you think.}

}

\begin{theorem} \label{cor:mixGauss}
Let $p$ be any distribution that has $\dtv(p,q) \leq \eps$ where $q$
is any one-dimensional mixture of $k$ Gaussians.
There is a $\poly(k,1/\eps)$-time algorithm that uses $\tilde{O}(k/\eps^2)$
samples and with high probability outputs a hypothesis $h$
such that $\dtv(h,p) \leq O(\eps).$
\end{theorem}

\noindent {It is straightforward to show that $\Omega(k/\eps^2)$ samples 
are information-theoretically necessary for learning a mixture of $k$
Gaussians, and thus our sample complexity is optimal up
to logarithmic factors.}

\medskip

\noindent {\bf Discussion.}   Moitra and Valiant \cite{MoitraValiant:10}
recently gave an algorithm for \emph{parameter estimation} (a stronger
requirement than the density estimation guarantees that we provide)
of any mixture of $k$ \emph{$n$-dimensional} Gaussians.  Their
algorithm has sample complexity that is exponential in $k$,
and indeed {they} prove that any algorithm that does
parameter estimation even for a mixture of $k$ one-dimensional Gaussians
must use $2^{\Omega(k)}$ samples.  In contrast, our result shows that
it is possible to perform \emph{density estimation} for any
mixture of $k$ one-dimensional Gaussians with a computationally
efficient algorithm that uses \emph{exponentially fewer}
(linear in $k$) samples than are required for parameter 
estimation.  Moreover,
unlike the parameter estimation results of \cite{MoitraValiant:10},
our density estimation algorithm is {semi}-agnostic:  it
succeeds even if the target distribution is $\eps$-far 
from a mixture of Gaussians.

\medskip

\noindent
{\bf Proof of Lemma~\ref{lem:mixGauss}:}
Without loss of generality we may take $p$ to be the standard Gaussian
$N(0,1),$ which has pdf
$p(x) = {\frac 1 {\sqrt 2 \pi}} e^{-x^2/2}.$
Let $I_1 = (-\infty,-C \sqrt{\log(1/\eps)}),$ $I_2 = 
[-C \sqrt{\log(1/\eps)},C\sqrt{\log(1/\eps)})$ and $I_3 = 
[C \sqrt{\log(1/\eps)},\infty)$
where $C>0$ is an absolute constant.
We define the distribution $q$ as follows:  $q(x)=0$ for all $x \in 
I_1 \cup I_3$, and $q(x)$ is given by the degree-$d$ Taylor expansion
of $p(x)$ about 0 for $x \in I_2$, where $d=O(\log(1/\eps)).$  
Clearly $q$ is a 3-piecewise degree-$d$ polynomial.
To see that $\dtv(p,q) \leq \eps$, we first observe that by a standard
Gaussian tail bound the regions $I_1$ and $I_3$ contribute at most $\eps/2$
to $\dtv(p,q)$ so it suffices to argue that
\begin{equation} \label{eq:int}
\int_{I_2} |p(x)-q(x)| dx \leq 
\eps/2.
\end{equation}

Fix any $x \in I_2$. Taylor's theorem gives that $|p(x)-q(x)| \leq 
p^{(d+1)}(x') x^{d+1}/(d+1)!$ for some $x' \in [0,x].$
Recalling that the $(d+1)$-st derivative $p^{(d+1)}(x')$ of the pdf of the
standard Gaussian equals
$H_{d+1}(x')p(x')$, where $H_{d+1}$ is the Hermite polynomial of
order $d+1$, standard bounds on the Hermite polynomials 
together with the fact that $|x| \leq C \sqrt{\log(1/\eps)}$ 
give that for $d = O(\log {\frac 1 \eps})$ we have
$|p(x) - q(x)| \leq \eps^2$ for all $x \in I_2$.
This gives the lemma.
\qed

\subsection{Learning discrete distributions.} \label{sec:discrete}

For convenience in this subsection we consider discrete distributions
over the $2N$-point finite domain 
\[
D := \left\{ - {\frac N N}, - {\frac {N-1} N}, \dots, - {\frac 1 N},
0, {\frac 1 N}, \dots, {\frac {N-1} N}\right\}.
\]
We say that a discrete distribution $q$ over domain $D$ is \emph{$t$-flat}
if there exists a partition of $D$ into $t$ intervals $I_1,\dots,I_t$
such that $q(i)=q(j)$ for all $i,j \in I_\ell$ for all $\ell=1,\dots,t.$
We say that a distribution $p$ over $D$ is 
\emph{$(\eps,t)$-flat} if $\dtv(p,q) \leq \eps$ for some
distribution $q$ over $D$ that is $t$-flat.

We begin by giving a simple reduction from learning $(\eps,t)$-flat
distributions over $D$ to learning $(\eps,t)$-piecewise
degree-0 distributions over $[-1,1].$
Together with Theorem~\ref{thm:main-detail} this reduction gives us an 
essentially optimal algorithm for learning discrete $(\eps,t)$-flat 
distributions (see Theorem~\ref{thm:opt-discrete}).
We then apply Theorem~\ref{thm:opt-discrete} to obtain highly efficient 
algorithms (in {most} cases with provably near-optimal sample complexity)
for various specific classes of discrete distributions {essentially resolving a number of open problems from previous works}.  

\subsubsection{A reduction from discrete to continuous.} 
\label{sec:disc-cont-red} 

Given a discrete distribution $p$ over $D$, we define
$\tilde{p}$ to be the distribution
over $[-1,1)$ defined as follows:  a draw from $\tilde{p}$ is obtained
by drawing a value $i/N$ from $p$, and then outputting $i + x/N$
where $x$ is distributed uniformly over $[0,1).$
It is easy to see that if distribution $p$ (over domain $D$)
is $t$-flat, then the distribution $\tilde{p}$ (over domain
$[-1,1)$) is $t$-piecewise degree-0.  Moreover, if
$p$ is $\tau$-close to some $t$-flat distribution $q$ over $D$,
then $\tilde{p}$ is $\tau$-close to $\tilde{q}$.

In the opposite direction, for $p$ a distribution over $[-1,1)$
we define $p^\ast$ to be the following distribution supported on $D$:
a draw from $p^\ast$ is obtained by sampling $x$ from $p$ and
then outputting the value obtained by rounding $x$ down to the next integer
multiple of $1/N.$
It is easy to see that if $p,q$ are distributions over $[-1,1)$ then
$\dtv(p,q)=\dtv(p^\ast,q^\ast).$
It is also clear that for $p$ a distribution over $D$ we have
$(\tilde{p})^\ast = p.$

With these relationships in hand, we may learn a $(\tau,t)$-flat
distribution $p$ over $D$ as follows:  run Algorithm
{\tt Learn-Piecewise-Poly}$(t,d=0,\eps)$ on the distribution
$\tilde{p}$.  Since $p$ is $(\tau,t)$-flat, $\tilde{p}$ is $\tau$-close
to some $t$-piecewise degree-0 distribution $q$ over $[-1,1)$, so
the algorithm with high probability constructs a hypothesis
$h$ over $[-1,1)$ such that 
$\dtv(h,\tilde{p}) \leq O(\tau + \eps)$.
The final hypothesis is $h^\ast$; for this hypothesis we have 
\[
\dtv(h^\ast,p) = \dtv(h^\ast,(\tilde{p})^\ast)=
\dtv(h,\tilde{p}) \leq O(\tau + \eps)
\]
as desired.

The above discussion and Theorem~\ref{thm:main-detail}
together give the following:

\begin{theorem} \label{thm:opt-discrete}
Let $p$ be a mixture of $k$ $(\tau,t)$-flat discrete distributions over $D$.
There is an algorithm which uses $\tilde{O}(kt/\eps^2)$ samples from $p$,
runs in time {$\poly(k, t, 1/\eps)$}, 
and with probability at least $9/10$ outputs a hypothesis distribution
$h$ over $D$ such that $\dtv(p,h) \leq O(\eps + \tau).$
\end{theorem}

We note that this is essentially a stronger version of Corollary~3.1 (the 
main technical result) of \cite{CDSS13soda}, which gave a similar guarantee
but with an algorithm that required $O(kt/\eps^3)$ samples.
{We also remark that $\Omega(kt/\eps^2)$ samples 
are information-theoretically required to learn an arbitrary  $k$-mixture 
of $t$-flat distributions. Hence, our sample complexity is optimal up to logarithmic factors (even for the case $\tau = 0$).

We would also like to mention the relation of the above theorem to a recent work by Indyk, Levi and Rubinfeld~\cite{ILR12}.
Motivated by a database application, \cite{ILR12} consider the problem of learning a $k$-flat distribution over $[n]$
{\em under the $L_2$ norm} and give an efficient algorithm that uses $O(k^2 \log (n) / \eps^4)$ samples.
Since the total variation distance is a stronger metric, Theorem~\ref{thm:opt-discrete} immediately implies an improved sample bound of 
$\tilde{O}(k/\eps^2)$ for their problem.  
}

\subsubsection{Learning specific classes of discrete distributions.}~
\label{sec:disc-ap}

\medskip

\noindent {\bf Mixtures of $t$-modal discrete distributions.}
Recall that
a distribution over an interval $I = [a,b] \cap D$
is said to be \emph{unimodal} if there is a value $y \in I$ such that
its pdf is monotone non-decreasing on $I \cap [-1,y]$ and monotone
non-increasing on $I \cap (y,1)$.
For $t>1$, a distribution $p$ over $D$ is $t$-modal if
there is a partition of $D$ into $t$ intervals $I_1,\dots,I_t$ such that
the conditional distributions $p_{I_1},\dots,p_{I_t}$ are each unimodal.

In \cite{CDSS13soda,DDSVV13soda} (building on \cite{Birge:87b})
it is shown that every $t$-modal distribution over $D$ is
$(\eps,t \log(N)/\eps)$-flat.  By using this fact together with 
Theorem~\ref{thm:opt-discrete} in place of Corollary~3.1 of \cite{CDSS13soda},
we improve the sample complexity of the \cite{CDSS13soda} algorithm
for learning mixtures of $t$-modal distributions and obtain the following:

\begin{theorem} \label{thm:mix-tmodal}
For any $t \geq 1$, let $p$ be any $k$-mixture of $t$-modal
distributions over $D$.  There is an algorithm that
runs in time $\poly(k,t,\log N, 1/\eps)$, 
draws $\tilde{O}(kt \log(N)/\eps^3)$ samples from $p$, and with probability
at least $9/10$ outputs a hypothesis distribution $h$ such that
$\dtv(p,h) \leq \eps$.  
\end{theorem}

We note that an easy adaptation of Birg\'{e}'s lower bound
\cite{Birge:87} for learning monotone distributions
(see the discussion at the end of Section~5 of \cite{CDSS13soda}) gives
that any algorithm for learning a $k$-mixture of $t$-modal
distributions over $D$ must use
$\Omega(k t \log(N/(kt))/\eps^3)$ samples, and hence the
sample complexity bound of Theorem~\ref{thm:mix-tmodal} is optimal
up to logarithmic factors.
We further note that even the $t=1$ case of this result compares 
favorably with the main result of \cite{DDS:12kmodallearn}, which gave
an algorithm for learning $t$-modal distributions over $D$ that uses
$O(t \log(N)/\eps^3) + \tilde{O}(t^3/\eps^3)$ samples.  
The \cite{DDS:12kmodallearn} result gave an optimal bound only
for small settings of $t$, specifically 
$t = \tilde{O}((\log N)^{1/3})$, and gave a 
quite poor bound as $t$ grows large; for example, at $t=(\log N)^2$
the optimal bound would be $O((\log N)^3/\eps^3)$ but the
\cite{DDS:12kmodallearn} result only gives $\tilde{O}((\log N)^9/\eps^3).$  
In contrast,
our new result gives an essentially optimal bound (up to log factors
in the optimal sample complexity) for \emph{all} settings of $t$.

\medskip

\noindent {\bf Mixtures of monotone hazard rate distributions.}
Let $p$ be a distribution supported on $D$. The \emph{hazard rate}
of $p$ is the function
$H(i) \eqdef {\frac {p(i)}{\littlesum_{j \geq i}
p(j)}}$; if $\littlesum_{j \geq i} p(j) = 0$ then we say $H(i) = +\infty.$
We say that $p$ has \emph{monotone hazard rate} (MHR) if
$H(i)$ is a non-decreasing function over $D.$

\cite{CDSS13soda} showed that every MHR distribution over $D$ 
is $(\eps,O(\log(N/\eps)/\eps))$-flat.
Theorem~\ref{thm:opt-discrete} thus gives us the following:

\begin{theorem} \label{thm:MHR}
Let $p$ be any $k$-mixture of MHR 
distributions over $D$.  There is an algorithm that
runs in time $\poly(k,\log N, 1/\eps)$, 
draws $\tilde{O}(k \log(N)/\eps^3)$ samples from $p$, and with probability
at least $9/10$ outputs a hypothesis distribution $h$ such that
$\dtv(p,h) \leq \eps$.  
\end{theorem}

In \cite{CDSS13soda} it is shown that any algorithm to learn $k$-mixtures
of MHR distributions over $D$ must use $\Omega(k \log(N/k)/\eps^3)$
samples, so Theorem~\ref{thm:MHR} is essentially optimal in its
sample complexity.

\medskip

\noindent {\bf Mixtures of discrete log-concave distributions.}
A probability distribution $p$ over $D$ is said to be \emph{log-concave}
if it satisfies the following conditions:
(i) if $i < j < k \in D$ are such that $p(i)p(k)>0$ then
$p(j) > 0$; and (ii) $p(k/N)^2 \geq p((k-1)/N)p((k+1)/N)$ for all $k \in 
\{-N+1,\dots,-1,0,1,\dots,N-2\}.$

In \cite{CDSS13soda} it is shown that every log-concave distribution
over $D$ is $(\eps,O(\log(1/\eps))/\eps)$-flat.  
Hence Theorem~\ref{thm:opt-discrete} gives:

\begin{theorem} \label{thm:logconcave}
Let $p$ be any $k$-mixture of log-concave 
distributions over $D$.  There is an algorithm that
runs in time $\poly(k,1/\eps)$, 
draws $\tilde{O}(k /\eps^3)$ samples from $p$, and with probability
at least $9/10$ outputs a hypothesis distribution $h$ such that
$\dtv(p,h) \leq \eps$.  
\end{theorem}

As in the previous examples, this improves the \cite{CDSS13soda} sample
complexity by essentially a factor of $1/\eps$.  We note that as a
special case of Theorem~\ref{thm:logconcave} we get an
efficient $O(k/\eps^3)$-sample algorithm for learning any mixture of
$k$ \emph{Poisson Binomial Distributions}.  (A Poisson Binomial
Distribution, or PBD, is a random variable of the form
$X_1 + \cdots + X_N$ where the $X_i$'s are independent
0/1 random variables that may have arbitrary and non-identical
means.)  The main result of \cite{DDS12stoc} gave an efficient
$\tilde{O}(1/\eps^3)$-sample algorithm for learning a single PBD;
here we achieve the same sample complexity, with an efficient algorithm,
for learning any mixture of any constant number of PBDs.

\ignore{
{We remark here that an extension of our Theorem~\ref{thm:main2} to the discrete setting for $d\ge 1$ would yield
an (essentially) optimal bound of $\tilde{O}(k /\eps^{5/2})$ for Theorem~\ref{thm:logconcave}.}
}

\bigskip

\noindent {\bf Acknowledgements.} We would like to thank Dany Leviatan for useful correspondence regarding his recent works~
\cite{KonL04, KonL07}. 

\bibliographystyle{alpha}
\bibliography{allrefs}

\newcommand{\etalchar}[1]{$^{#1}$}
\begin{thebibliography}{KMR{\etalchar{+}}94}

\bibitem[AK03]{AK03}
Sanjeev Arora and Subhash Khot.
\newblock Fitting algebraic curves to noisy data.
\newblock {\em J. Comput. Syst. Sci.}, 67(2):325--340, 2003.

\bibitem[An95]{An:95}
M.~Y. An.
\newblock Log-concave probability distributions: Theory and statistical
  testing.
\newblock Technical Report Economics Working Paper Archive at WUSTL, Washington
  University at St. Louis, 1995.

\bibitem[Ass83]{Assouad:83}
P.~Assouad.
\newblock {Deux remarques sur l'estimation}.
\newblock {\em C. R. Acad. Sci. Paris S\'{e}r. I}, 296:1021--1024, 1983.

\bibitem[BBBB72]{BBBB:72}
R.E. Barlow, D.J. Bartholomew, J.M. Bremner, and H.D. Brunk.
\newblock {\em Statistical Inference under Order Restrictions}.
\newblock Wiley, New York, 1972.

\bibitem[Bir87a]{Birge:87}
L.~Birg\'e.
\newblock {Estimating a density under order restrictions: Nonasymptotic minimax
  risk}.
\newblock {\em Annals of Statistics}, 15(3):995--1012, 1987.

\bibitem[Bir87b]{Birge:87b}
L.~Birg\'e.
\newblock {On the risk of histograms for estimating decreasing densities}.
\newblock {\em Annals of Statistics}, 15(3):1013--1022, 1987.

\bibitem[Bru58]{Brunk:58}
H.~D. Brunk.
\newblock On the estimation of parameters restricted by inequalities.
\newblock {\em The Annals of Mathematical Statistics}, 29(2):pp. 437--454,
  1958.

\bibitem[BRW09]{BRW:09aos}
F.~Balabdaoui, K.~Rufibach, and J.~A. Wellner.
\newblock Limit distribution theory for maximum likelihood estimation of a
  log-concave density.
\newblock {\em The Annals of Statistics}, 37(3):pp. 1299--1331, 2009.

\bibitem[BS10]{BelkinSinha:10}
M.~Belkin and K.~Sinha.
\newblock Polynomial learning of distribution families.
\newblock In {\em FOCS}, pages 103--112, 2010.

\bibitem[BW07]{BW07aos}
F.~Balabdaoui and J.~A. Wellner.
\newblock Estimation of a $k$-monotone density: Limit distribution theory and
  the spline connection.
\newblock {\em The Annals of Statistics}, 35(6):pp. 2536--2564, 2007.

\bibitem[BW10]{BW10sn}
F.~Balabdaoui and J.~A. Wellner.
\newblock Estimation of a $k$-monotone density: characterizations, consistency
  and minimax lower bounds.
\newblock {\em Statistica Neerlandica}, 64(1):45--70, 2010.

\bibitem[CDSS13]{CDSS13soda}
S.~Chan, I.~Diakonikolas, R.~Servedio, and X.~Sun.
\newblock Learning mixtures of structured distributions over discrete domains.
\newblock In {\em SODA}, 2013.

\bibitem[DDS12a]{DDS:12kmodallearn}
C.~Daskalakis, I.~Diakonikolas, and R.A. Servedio.
\newblock Learning $k$-modal distributions via testing.
\newblock In SODA, 2012.

\bibitem[DDS12b]{DDS12stoc}
C.~Daskalakis, I.~Diakonikolas, and R.A. Servedio.
\newblock {Learning Poisson Binomial Distributions}.
\newblock In {\em STOC}, pages 709--728, 2012.

\bibitem[DDS{\etalchar{+}}13]{DDSVV13soda}
C.~Daskalakis, I.~Diakonikolas, R.~Servedio, G.~Valiant, and P.~Valiant.
\newblock Testing $k$-modal distributions: Optimal algorithms via reductions.
\newblock In {\em SODA, to appear}, 2013.

\bibitem[DG85]{DG85}
L.~Devroye and L.~Gy\"{o}rfi.
\newblock {\em {Nonparametric Density Estimation: The $L_1$ View}}.
\newblock John Wiley \& Sons, 1985.

\bibitem[DGJ{\etalchar{+}}10]{DGJ+10:bifh}
I.~Diakoniokolas, P.~Gopalan, R.~Jaiswal, R.~Servedio, and E.~Viola.
\newblock Bounded independence fools halfspaces.
\newblock {\em SIAM Journal on Computing}, 39(8):3441--3462, 2010.

\bibitem[DL01]{DL:01}
L.~Devroye and G.~Lugosi.
\newblock {\em Combinatorial methods in density estimation}.
\newblock Springer Series in Statistics, Springer, 2001.

\bibitem[DR09]{DumbgenRufibach:09}
L.~D\:{u}mbgen and K.~Rufibach.
\newblock Maximum likelihood estimation of a log-concave density and its
  distribution function: Basic properties and uniform consistency.
\newblock {\em Bernoulli}, 15(1):40--68, 2009.

\bibitem[Dud74]{Dudley:74}
R.M Dudley.
\newblock Metric entropy of some classes of sets with differentiable
  boundaries.
\newblock {\em Journal of Approximation Theory}, 10(3):227 -- 236, 1974.

\bibitem[FM99]{FreundMansour:99}
Y.~Freund and Y.~Mansour.
\newblock Estimating a mixture of two product distributions.
\newblock In {\em Proceedings of the Twelfth Annual Conference on Computational
  Learning Theory}, pages 183--192, 1999.

\bibitem[FOS05]{FOS:05focs}
J.~Feldman, R.~O'Donnell, and R.~Servedio.
\newblock Learning mixtures of product distributions over discrete domains.
\newblock In {\em Proc.\ 46th Symposium on Foundations of Computer Science
  (FOCS)}, pages 501--510, 2005.

\bibitem[Gre56]{Grenander:56}
U.~Grenander.
\newblock On the theory of mortality measurement.
\newblock {\em Skand. Aktuarietidskr.}, 39:125--153, 1956.

\bibitem[Gro85]{Groeneboom:85}
P.~Groeneboom.
\newblock Estimating a monotone density.
\newblock In {\em Proc. of the Berkeley Conference in Honor of Jerzy Neyman and
  Jack Kiefer}, pages 539--555, 1985.

\bibitem[GW09]{GW09sc}
F.~Gao and J.~A. Wellner.
\newblock On the rate of convergence of the maximum likelihood estimator of a
  $k$-monotone density.
\newblock {\em Science in China Series A: Mathematics}, 52:1525--1538, 2009.

\bibitem[Ham87]{Hampel87}
F.~R. Hampel.
\newblock Design, data \& analysis.
\newblock chapter Design, modelling, and analysis of some biological data sets,
  pages 93--128. John Wiley \& Sons, Inc., New York, NY, USA, 1987.

\bibitem[HP76]{HansonP:76}
D.~L. Hanson and G.~Pledger.
\newblock Consistency in concave regression.
\newblock {\em The Annals of Statistics}, 4(6):pp. 1038--1050, 1976.

\bibitem[ILR12]{ILR12}
P.~Indyk, R.~Levi, and R.~Rubinfeld.
\newblock {Approximating and Testing $k$-Histogram Distributions in Sub-linear
  Time}.
\newblock In {\em PODS}, pages 15--22, 2012.

\bibitem[Jac97]{Jackson:97}
J.~Jackson.
\newblock An efficient membership-query algorithm for learning {DNF} with
  respect to the uniform distribution.
\newblock {\em Journal of Computer and System Sciences}, 55:414--440, 1997.

\bibitem[KL04]{KonL04}
V.~N. Konovalov and D.~Leviatan.
\newblock Free-knot splines approximation of $s$-monotone functions.
\newblock {\em Adv. Comput. Math.}, 20(4):347--366, 2004.

\bibitem[KL07]{KonL07}
V.~N. Konovalov and D.~Leviatan.
\newblock Freeknot splines approximation of sobolev-type classes of $s$
  -monotone functions.
\newblock {\em Adv. Comput. Math.}, 27(2):211--236, 2007.

\bibitem[KM93]{KushilevitzMansour:93}
E.~Kushilevitz and Y.~Mansour.
\newblock Learning decision trees using the {F}ourier spectrum.
\newblock {\em SIAM J. on Computing}, 22(6):1331--1348, 1993.

\bibitem[KM10]{KoenkerM:10aos}
R.~Koenker and I.~Mizera.
\newblock Quasi-concave density estimation.
\newblock {\em Ann. Statist.}, 38(5):2998--3027, 2010.

\bibitem[KMR{\etalchar{+}}94]{KMR+:94}
M.~Kearns, Y.~Mansour, D.~Ron, R.~Rubinfeld, R.~Schapire, and L.~Sellie.
\newblock On the learnability of discrete distributions.
\newblock In {\em Proceedings of the 26th Symposium on Theory of Computing},
  pages 273--282, 1994.

\bibitem[KMV10]{KMV:10}
A.~T. Kalai, A.~Moitra, and G.~Valiant.
\newblock {Efficiently learning mixtures of two Gaussians}.
\newblock In {\em STOC}, pages 553--562, 2010.

\bibitem[KOS04]{KOS:04}
A.~Klivans, R.~O'Donnell, and R.~Servedio.
\newblock Learning intersections and thresholds of halfspaces.
\newblock {\em Journal of Computer \& System Sciences}, 68(4):808--840, 2004.

\bibitem[KS04]{KlivansServedio:04jcss}
A.~Klivans and R.~Servedio.
\newblock {Learning {DNF} in time {$2^{\tilde{O}(n^{1/3})}$}}.
\newblock {\em Journal of Computer \& System Sciences}, 68(2):303--318, 2004.

\bibitem[LMN93]{LMN:93}
N.~Linial, Y.~Mansour, and N.~Nisan.
\newblock Constant depth circuits, {F}ourier transform and learnability.
\newblock {\em Journal of the ACM}, 40(3):607--620, 1993.

\bibitem[MOS04]{MOS:04}
E.~Mossel, R.~O'Donnell, and R.~Servedio.
\newblock Learning functions of $k$ relevant variables.
\newblock {\em {Journal of Computer \& System Sciences}}, 69(3):421--434, 2004.
\newblock Preliminary version in \emph{Proc. STOC'03}.

\bibitem[MR95]{MotwaniRaghavan:95}
R.~Motwani and P.~Raghavan.
\newblock {\em Randomized Algorithms}.
\newblock Cambridge University Press, New York, NY, 1995.

\bibitem[MV10]{MoitraValiant:10}
A.~Moitra and G.~Valiant.
\newblock {Settling the polynomial learnability of mixtures of Gaussians}.
\newblock In {\em FOCS}, pages 93--102, 2010.

\bibitem[Nov88]{Novak:88}
E.~Novak.
\newblock {\em Deterministic and Stochastic Error Bounds In Numerical
  Analysis}.
\newblock Springer-Verlag, 1988.

\bibitem[PA13]{PA13cgs}
D.~Papp and F.~Alizadeh.
\newblock Shape constrained estimation using nonnegative splines.
\newblock {\em Journal of Computational and Graphical Statistics}, 0(ja):null,
  2013.

\bibitem[Rao69]{PrakasaRao:69}
B.L.S.~Prakasa Rao.
\newblock Estimation of a unimodal density.
\newblock {\em Sankhya Ser. A}, 31:23--36, 1969.

\bibitem[Reb05]{Reb05aos}
L.~Reboul.
\newblock Estimation of a function under shape restrictions. {A}pplications to
  reliability.
\newblock {\em Ann. Statist.}, 33(3):1330--1356, 2005.

\bibitem[Sco92]{Scott:92}
D.W. Scott.
\newblock {\em Multivariate Density Estimation: Theory, Practice and
  Visualization}.
\newblock Wiley, New York, 1992.

\bibitem[Ser10]{S10ams}
A.~Seregin.
\newblock {Uniqueness of the maximum likelihood estimator for $k$-monotone
  densities}.
\newblock {\em Proceedings of The American Mathematical Society},
  138:4511--4511, 2010.

\bibitem[Sil86]{Silverman:86}
B.~W. Silverman.
\newblock {\em Density Estimation}.
\newblock Chapman and Hall, London, 1986.

\bibitem[Wal09]{Walther09}
G.~Walther.
\newblock Inference and modeling with log-concave distributions.
\newblock {\em Statistical Science}, 24(3):319--327, 2009.

\bibitem[Weg70]{Wegman:70}
E.J. Wegman.
\newblock {Maximum likelihood estimation of a unimodal density. I. and II.}
\newblock {\em Ann. Math. Statist.}, 41:457--471, 2169--2174, 1970.

\end{thebibliography}

\appendix

\section{Omitted proofs} \label{ap:z}

\subsection{Proof of Lemma~\ref{lem:part-approx-unif}.}
Recall Lemma~\ref{lem:part-approx-unif}:

\medskip

\noindent {\bf Lemma~\ref{lem:part-approx-unif}.}
\emph{
Given $0 < \kappa< 1$ and access to samples from an $\kappa/64$-well-behaved
distribution $p$ over $[-1,1)$,
the procedure {\tt Approximately-Equal-Partition} uses $\tilde{O}(1/\kappa)$
samples from $p$, runs in time $\tilde{O}(1/\kappa)$, and
with probability at least $99/100$ outputs a partition of
$[-1,1)$ into $\ell=\Theta(1/\kappa)$ intervals such that $p(I_j)
\in [{\frac 1 {2 \kappa}}, {\frac 3 \kappa}]$ for all $1 \leq j \leq \ell.$
}

\medskip

\noindent {\bf Proof of Lemma~\ref{lem:part-approx-unif}:}
Let $n$ denote $1/\kappa$ (we assume wlog that $n$ is an integer).
Let $S$ be a sample of $m = \Theta(n \log n)$ i.i.d. draws from $p$,
where $m$ is an integer multiple of $n$.
For $1 \leq i \leq m$
let $U_{(i)}$ denote the $i$-th order statistic of $S$, i.e.
the $i$-th largest element of $S$.
Let $U_{(0)} := -1.$

Our goal is to show that with high probability, for each
$j \in \{1,\dots,n\}$ we have
$p([U_{({\frac {j-1} n} \cdot m)}, U_{({\frac j n} \cdot m)})) \in
[{\frac 1 {2n}}, {\frac 2 n}]$.
This means that simply greedily taking the intervals $I_1$,
$I_2,\dots$ from left to right, where the left endpoint
of $I_0$ is $-1$, the left (closed) endpoint of the $j$-th
interval is the same as the right (open)
endpoint of the $(j-1)$st interval, and the $j$-th
interval ends at $U_{({\frac j n} \cdot m)}$,
the resulting intervals have
probability masses as desired.
(These intervals cover $[-1,U_{(m)}]$; an easy argument shows that
with probability at least $1 - 1/n$, the uncovered
region $(U_{(m)},1)$ has mass at most $1/n$ under $p$ , so we may
add it to the final interval.)

Let $P$ denote the cumulative density functions associated with $p$.
For $0 \leq \alpha < \beta \leq 1$
let $\#_S[\alpha,\beta)$ denote the number of elements $x \in S$ that have
$P(x) \in [\alpha,\beta).$
A multiplicative Chernoff bound and a union bound together straightforwardly
give that with
probability at least $99/100$, for each $i \in \{1,\dots,8n\}$ we have
$\#_S[{\frac {i-1} {8n}}, {\frac i {8n}}) \in [{\frac 1 {16}} \cdot {\frac m n},
\new{{\frac 3 {16}}} \cdot {\frac m n}].$
(Note that since $p$ is ${\frac 1 {64n}}$-well-behaved, the amount of mass
that $p$ puts on $P^{-1}([{\frac {i-1} {8n}},{\frac i
{8n}}))$ lies in $[{\frac 3 {32n}},{\frac 5 {32n}}].$)
As an immediate consequence of this we get that
$p([U_{({\frac {j-1} n} \cdot m)}, U_{({\frac j n} \cdot m)}]) \in
[{\frac 1 {2n}}, {\frac 2 n}]$
for each $j \in \{1,\dots,n\}$, which establishes the lemma.
\qed

\new{

\subsection{Proof of Theorem~\ref{thm:lower-bound-precise}.}
\label{ap:lower}
Recall Theorem~\ref{thm:lower-bound-precise}:

\medskip

\noindent {\bf Theorem~\ref{thm:lower-bound-precise}.}
\emph{Let $p$ be an unknown $t$-piecewise degree-$d$ distribution over
$[-1,1)$
where $t\geq 1,$ $d \geq 0$ satisfy $t+d > 1.$
Let $L$ be any algorithm which, given as input $t,d,\eps$ and access to independent samples from $p$,
outputs a hypothesis distribution $h$ such that
$\E[\dtv(p,h)] \leq \eps$, where the expectation is over  
the random samples drawn from $p$ and any internal randomness of   
$L$.  Then $L$ must use at
least $\Omega({\frac {t(d+1)}{(1+\log (d+1))^2}} \cdot {\frac 1 {\eps^2}})$
samples.}

\medskip

We first observe that if $d=0$ then the claimed 
$\Omega(t/\eps^2)$ lower bound follows easily from the standard
fact that this many samples are required to learn
an unknown distribution over the $t$-element
set $\{1,\dots,t\}$.  (This fact follows easily from Assouad's lemma; 
we will essentially prove it using Assouad's lemma 
in Section~\ref{sec:idea} below.) 
Thus we may assume below that $d > 0$; in fact,
we can (and do) assume that $d \geq C$ where $C$ may be taken to be
any fixed absolute constant.

In what follows we shall use Assouad's lemma to establish an
$\Omega({\frac {d}{(\log d)^2}} \cdot {\frac 1 {\eps^2}})$
lower bound for learning a single degree-$d$ distribution over $[-1,1)$
to accuracy $\eps$.  The same argument applied to a concatenation of $t$
equally weighted copies of this lower bound construction 
over $t$ disjoint intervals $[-1, -1 + {\frac 2 t}),
\dots, [1 - {\frac 2 t},1)$ (again using Assouad's lemma) yields
Theorem~\ref{thm:lower-bound-precise}.  Thus to prove
Theorem~\ref{thm:lower-bound-precise} for general $t$ it is enough to
prove the following lower bound, corresponding to $t=1$.
(For ease of exposition in our later arguments, we take the
domain of $p$ below to be the interval $[0,2k)$ rather than $[-1,1).$)

\begin{theorem}
\label{thm:lower-bound-d-is-1}
Fix an integer $d \geq C$.
Let $p$ be an unknown degree-$d$ distribution over
$[0,2k)$.  Let $L$ be any algorithm which, given as input
$d,\eps$ and access to independent samples from $p$,
outputs a hypothesis distribution $h$ such that $\E[\dtv(p,h)] \leq \eps$.  
Then $L$ must use at
least $\Omega({\frac {d}{(\log d)^2}} \cdot {\frac 1 {\eps^2}})$
samples.
\end{theorem}

Our main tool for proving Theorem~\ref{thm:lower-bound-d-is-1}
is Assouad's Lemma \cite{Assouad:83}.
We recall the statement of Assouad's Lemma from \cite{DG85} below.
(The statement below is slightly tailored to our context, in that we have
taken the underlying domain to be $[0,2k)$ and the partition of the domain
to be $[0,2), [2,4), \dots, [2k-2, 2k)$.)

\begin{theorem} \label{thm:assouad}
[Theorem 5, Chapter 4, \cite{DG85}]
Let $k \geq 1$ be an integer.   
For each $b = (b_1,\dots,b_k) \in \{-1,1\}^k$, let $p_b$ be a probability
distribution over $[0,2k)$.

Suppose that the distributions $p_b$ satisfy the following properties:
Fix any $\ell \in [k]$ and any $b \in \{-1,1\}^k$ with $b_\ell=1$.
Let $b' \in \{-1,1\}^k$ be the same as $b$ but with $b'_\ell=-1.$
The properties are that

\begin{enumerate}

\item
$\int_{2\ell-2}^{2\ell} |p_b(x) - p_{b'}(x)| dx \geq \alpha$, and

\item
$\int_{0}^{2k} \sqrt{p_b(x) p_{b'}(x)} dx \geq 1-\gamma > 0.$

\end{enumerate}

Then for any any algorithm $L$ that draws $n$ samples from an
unknown $p \in \{p_b\}_{b \in \{-1,1\}^k}$ and outputs a hypothesis
distribution $h$, there is some $b \in \{-1,1\}^k$
such that if the target distribution $p$ is $p_b$, then
\begin{equation} \label{eq:assouad}
\E[\dtv(p_b,h)] \geq (k \alpha/4)(1 - \sqrt{2 n \gamma}).
\end{equation}
\end{theorem}

We will use this lemma in the following way:
Fix any $d \geq C$ and any $0 < \eps < 1/2.$  
We will exhibit a family
of $2^k$ distributions $p_b$, where each $p_b$ is a degree-$d$ polynomial
distribution and $k = \Theta(d / (\log d)^2).$  
We will show that all pairs $b,b' \in \{-1,1\}^k$
as specified in Theorem~\ref{thm:assouad} satisfy condition (1) with 
$\alpha = \Omega(\eps/k)$, and satisfy condition (2) with $\gamma = 
O(\eps^2/k).$
With these conditions, consider an algorithm $L$ that draws 
$n = 1/(8 \gamma)$ samples from the unknown target
distribution $p$.  The right-hand side of (\ref{eq:assouad}) 
simplifies to $k \alpha / 8 = \Omega(\eps)$, and hence 
by Theorem~\ref{thm:assouad}, the expected variation distance error of
algorithm $L$'s hypothesis $h$ is $\Omega(\eps).$  This 
yields Theorem~\ref{thm:lower-bound-d-is-1}.

Thus, in the rest of this subsection, to prove 
Theorem~\ref{thm:lower-bound-d-is-1} and thus establish 
Theorem~\ref{thm:lower-bound-precise}, it suffices for us to describe
the $2^k$ distributions $p_b$ and establish conditions (1) and (2) 
with the claimed bounds $\alpha = \Omega(\eps/k)$ and $\gamma =
O(\eps^2/k).$  We do this below.

\subsubsection{The idea behind the construction.} \label{sec:idea}
We provide some intuition before entering into the details of our
construction.  Intuitively, each polynomial $p_b$ (for a given
$b \in \{-1,1\}^k$) is an approximation, over the interval
$[0,2k)$ of interest, of a $2k$-piecewise constant distribution
$S_b$ that we describe below.  To do this, first let us
define the $2k$-piecewise constant distribution
\[
R_b(x) = R_{b,1}(x) + ... + R_{b,k}(x)
\]
over $[0,2k)$,
where $R_{b,i}(x)$ is a function which is 0 outside of the 
interval $[2i-2, 2i).$
For $x \in [2i-2,2i-1)$ we have $R_{b,i}(x) = (1 + b_i \cdot \eps)/(2k)$, and
for $x \in [2i-1,2i)$ we have $R_{b,i}(x) = (1 - b_i \cdot \eps)/(2k).$
So note that regardless of whether $b_i$ is 1 or $-1$, we
have $\int_{2i-2}^{2i} R_{b,i}(x) dx = 1/k$ and hence
$\int_{0}^{2k} R_b(x) dx = 1$, so $R_b$ is indeed a 
probability distribution over the domain $[0,2k).$

The distribution $S_b$ over $[0,2k)$ is defined as 
\begin{equation} \label{eq:Qb}
S_b(x) 
= {\frac 1 {10}} \cdot {\frac 1 {2k}} + {\frac 9 {10}} \cdot R_b(x).
\end{equation}
(The reason for ``mixing'' $R_b$ with the uniform distribution
will become clear later; roughly, it is to control the adverse effect
on condition (2) of having only a polynomial approximation $p_b$ instead of 
the actual piecewise constant distribution.)

To motivate the goal of constructing polynomials $p_b$ that approximate
the piecewise constant distributions $S_b$, 
let us verify that the distributions $\{S_b\}_{b \in \{-1,1\}^k}$ 
satisfy conditions (1) and (2) of Theorem~\ref{thm:assouad} with the 
desired parameters.  So fix any $b \in \{-1,1\}^k$ with $b_\ell=1$
and let $b' \in \{-1,1\}^k$ differ from $b$ precisely in the $\ell$-th
coordinate.
For (1), we immediately have that
\[
\int_{2\ell-2}^{2\ell} |S_b(x) - S_{b'}(x)| dx =
{\frac 9 {10}}\int_{2\ell-2}^{2\ell}|R_{b,\ell}(x) - R_{b',\ell}(x)| dx 
= {\frac 9 5} \cdot {\frac \eps k}.
\]
For (2), we have that for any two distributions $f,g$,
\begin{eqnarray*}
\int_{0}^{2k}\sqrt{f(x)g(x)} dx = 1 - h(f,g)^2
\end{eqnarray*}
where $h(f,g)^2$ is the squared Hellinger distance between
$f$ and $g$,
\begin{eqnarray*}
h(f,g)^2 = {\frac 1 2} \int_{0}^{2k} \left(
\sqrt{f(x)} - \sqrt{g(x)}\right)^2 dx.
\end{eqnarray*}
Applying this to $S_b$ and $S_{b'}$, we get
\begin{eqnarray*}
h(S_b,S_{b'})^2
&=& {\frac 1 2} \int_{0}^{2k} \left(\sqrt{S_b(x)} - \sqrt{S_{b'}(x)}\right)^2 
dx\\
&=& {\frac 1 2} \int_{2\ell-2}^{2\ell} 
\left(
\sqrt{{\frac 1 {20k}} + {\frac 9 {10}} \cdot {\frac {1 + \eps}{2k}}} - 
\sqrt{{\frac 1 {20k}} + {\frac 9 {10}} \cdot {\frac {1 - \eps}{2k}}} 
\right)^2
dx\\
&=& \Theta(\eps^2/k),
\end{eqnarray*}
as desired.
We now turn to the actual construction.

\subsubsection{The construction.}
Fix any $b \in \{-1,1\}^k$.  Our goal is to give a degree-$d$ polynomial
$p_b$ that is a high-quality approximator of $S_b(x)$ over $[0,2k).$
We shall do this by approximating each $R_{b,i}(x)$ and combining the
approximators in the obvious way.

We can write each $R_{b,i}(x)$ as $R_{b,i,1}(x) + R_{b,i,2}(x)$ where 
$R_{b,i,1}(x)$ is 0 outside of $[2i-2,2i-1)$ and $R_{b,i,2}(x)$ is 0 
outside of $[2i-1,2i).$
So $R_b(x)$ is the sum of $2k$ many functions each of which is of the form
$\omega_{b,j} \cdot I_j(x),$ i.e.

\begin{equation} \label{eq:exact}
R_b(x) = \sum_{j=1}^{2k} \omega_{b,j} \cdot I_j(x)
\end{equation}
where each $\omega_{b,j}$ is either $(1+\eps)/2k$ or is $(1-\eps)/2k$
and $I_j$ is the indicator function of the interval $[j-1,j)$:  i.e.
$I_j(x)=1$ if $x \in [j-1,j)$ and is 0 elsewhere.

We shall approximate each indicator function $I_j(x)$ over $[0,2k)$ 
by a low-degree univariate polynomial which we shall denote 
$\tilde{I}_j(x)$; then we will multiply each $\tilde{I}_j(x)$
by $\omega_{b,j}$ and sum the results to obtain our polynomial
approximator $\tilde{R}_b(x)$ to $R_b(x),$ i.e.
\begin{equation} \label{eq:approx}
\tilde{R}_b(x) := \sum_{j=1}^{2k} \omega_{b,j} \tilde{I}_j(x).
\end{equation}

The starting point of our construction is the polynomial 
whose existence is asserted in Lemma~3.7 of 
\cite{DGJ+10:bifh}; this is essentially a low-degree univariate
polynomial which is a high-accuracy approximator to the function 
$\sign(x)$ over $[-1,1)$ except for values of $x$ that have small
absolute value.
Taking $k = M \log(1/\eps)$ in Claim~3.8 of
\cite{DGJ+10:bifh} for $M$ a sufficiently large constant (rather than 
$M=15$ as is done in \cite{DGJ+10:bifh}), the construction employed in the
proof of Lemma~3.7 gives the following:

\begin{fact} \label{fact:first-poly}
For $0 \leq \tau \leq c$, where $c<1$ is an absolute
constant, there is a polynomial $A(x)$ of degree $O((\log(1/\tau))^2/\tau)$
such that

\begin{enumerate}

\item For all $x \in [-1,-\tau)$ we have $A(x) \in [-1, -1 + \tau^{10}]$;

\item For all $x \in (\tau,1]$ we have $A(x) \in [1 - \tau^{10},1]$;

\item For all $x \in [-\tau,\tau]$ we have $A(x) \in [-1,1].$

\end{enumerate}

\end{fact}

For $-1/4 \leq \theta \leq 1/4$ let $B_\theta(x)$ denote
the polynomial $B_\theta(x) = (A(x)-A(x-\theta))/2$.
Given Fact~\ref{fact:first-poly}, it is easy to see that 
$B_\theta(x)$
has degree $O((\log(1/\tau))^2/\tau)$ and, over the 
interval $[-1/2,1/2]$, is a high-accuracy approximation to the
indicator function of the interval $[0,\theta]$ 
except on ``error regions'' of width at most $\tau$ at
each of the endpoints $0,\theta$.

Next, recall that $k = \Theta(d/(\log d)^2)$ where
$d$ is at least some universal constant $C$.  Choosing $\tau = \delta/k$
for a suitably small positive absolute constant $\delta$, and performing
a suitable linear scaling and shifting of the polynomial $B_\theta(x)$, we get
the following:

\begin{fact} \label{fact:second-poly}
Fix any integer $1 \leq j \leq 2k$.  There is a polynomial $C_j(x)$
of degree at most $d$ which is such that

\begin{enumerate}

\item For $x \in [j-0.999,j-0.001)$ we have $C_j(x) \in [1-1/k^{5},1)]$;
\item For $x \in [0,j-1) \cup [j,2k)$ we have $C_j(x) \in [0,1/k^{5}]$;
\item For $x \in [j-1,j-0.999) \cup [j-0.001,j)$ we have 
$0 \leq C_j(x) \leq 1$.

\end{enumerate}

\end{fact}

The desired polynomial $\tilde{I}_j(x)$ which is an approximator of the 
indicator function $I_j(x)$ is obtained by renormalizing $C_j$ so that
it integrates to $1$ over the domain $[0,2k)$; i.e. we define

\begin{equation} \label{eq:tildeI-def}
\tilde{I}_j(x) = C_j(x)/\int_{0}^{2k} C_j(x)dx.
\end{equation}

By Fact~\ref{fact:second-poly} we have that
$\int_{0}^{2k} C_j(x)dx \in [0.997,1.003]$, and thus we obtain
the following:

\begin{fact} \label{fact:third-poly}
Fix any integer $1 \leq j \leq 2k.$  The polynomial $\tilde{I}_j(x)$
has degree at most $d$ and is such that

\begin{enumerate}

\item For $x \in [j-0.999,j-0.001)$ we have $\tilde{I}_j(x) \in [0.996,1.004)]$;

\item For $x \in [0,j-1) \cup [j,2k)$ we have $\tilde{I}_j(x) \in
[0,1/k^{4}]$;

\item For $x \in [j-1,j-0.999) \cup [j-0.001,j)$ we have $0 \leq 
\tilde{I}_j \leq 1.004$; and

\item $\int_{0}^{2k} \tilde{I}_j(x) dx = 1.$

\end{enumerate}

\end{fact}

Recall that from~(\ref{eq:approx}) the polynomial approximator
$\tilde{R}_b(x)$ for $R_b(x)$ is defined as
$\tilde{R}_b(x) = \sum_{j=1}^{2k} \omega_{b,j} \tilde{I}_j(x).$  
We define the final polynomial $p_b(x)$ as

\begin{equation} \label{eq:pb}
p_b(x) = 
{\frac 1 {10}} \cdot {\frac 1 {2k}} + {\frac 9 {10}} \cdot 
\tilde{R}_b(x).
\end{equation}

Since $\sum_{j=1}^{2k} \omega_{b,j} = 1$ for every
$b \in \{-1,1\}^k$, the polynomial $p_b$ does
indeed define a legitimate probability distribution over $[0,2k).$

It will be useful for us to take the following alternate view on $p_b(x)$.
Define

\begin{equation} \label{eq:Jj}
\tilde{J}_j(x) = {\frac 1 {10}} \cdot {\frac 1 {2k}} + 
{\frac 9 {10}} \cdot \tilde{I}_j(x).
\end{equation}
Recalling that $\sum_{j=1}^{2k} \omega_{b,j} = 1$, we may alternately
define $p_b$ as

\begin{equation} \label{eq:pb-alt}
p_b(x) = \sum_{j=1}^{2k} \omega_{b,j} \tilde{J}_j(x).
\end{equation}

The following is an easy consequence of Fact~\ref{fact:third-poly}:

\begin{fact} \label{fact:fourth-poly}
Fix any $1 \leq j \leq 2k.$  The polynomial $\tilde{J}_j(x)$
has degree at most $d$ and is such that

\begin{enumerate}

\item For $x \in [j-0.999,j-0.001)$ we have $\tilde{J}_j(x) \in
[0.896 + 0.1/(2k),0.9004 + 0.1/(2k)]$;

\item For $x \in [0,j-1) \cup [j,2k)$ we have $\tilde{I}_j(x) \in
[0.1/(2k),0.1/(2k) + 1/k^{4}]$;

\item For $x \in [j-1,j-0.999) \cup [j-0.001,j)$ we have $\tilde{J}_j(x)
\in [0.1/(2k),0.9004 + 0.1/(2k))$ ;  and

\item $\int_{0}^{2k} \tilde{J}_j(x) dx = 1.$

\end{enumerate}

\end{fact}

We are now ready to prove that the distributions $\{p_b\}_{b \in \{-1,1\}^k}$
satisfy properties (1) and (2) of Assouad's lemma with $\alpha = 
\Omega(\eps/k)$ and $\gamma =O(\eps^2/k)$ as described in 
the discussion following Theorem~\ref{thm:assouad}.
Fix $b \in \{-1,1\}^k$ with $b_\ell = 1$ and $b' \in \{-1,1\}^k$
which agrees with $b$ except in the $\ell$-th coordinate.
We establish properties (1) and (2) in the following two
claims:

\begin{claim} \label{claim:1}
We have 
$\int_{2\ell-2}^{2\ell} |p_b(x) - p_{b'}(x)| dx \geq  \Omega(\eps/k)$.
\end{claim}

\begin{proof}
Recall from (\ref{eq:pb-alt}) that

\[
p_b(x) = \sum_{j=1}^{2k} \omega_{b,j}\cdot\tilde{J}_j(x)
\quad
\text{
and
}
\quad
p_{b'}(x) = \sum_{j=1}^{2k} \omega_{b',j}\cdot \tilde{J}_j(x).
\]

We have that $\omega_{b,j}=\omega_{b',j}$ for all but exactly two
(adjacent) values of $j$, which are $j=2\ell-1$ and $j=2\ell.$
For those values we have

\[
\omega_{b,2\ell-1} =   (1+\eps)/(2k), \quad \omega_{b',2\ell-1} = (1-\eps)/(2k)
\]

while
\[
\omega_{b,2\ell} = (1-\eps)/(2k), \quad \omega_{b',2\ell} = (1+\eps)/(2k).
\]

So 
we have

\begin{eqnarray*}
\int_{2\ell-2}^{2\ell} |p_b(x) - p_{b'}(x)| dx 
&=& \int_{2\ell-2}^{2\ell}
|
(\omega_{b,2\ell-1}\tilde{J}_{2\ell-1}(x) + 
 \omega_{b,2\ell}\tilde{J}_{2\ell}(x)) - 
(\omega_{b',2\ell-1}\tilde{J}_{2\ell-1}(x) + 
 \omega_{b',2\ell} \tilde{J}_{2\ell}(x))| dx\\
&=& (\eps/k)\cdot \int_{2\ell-2}^{2\ell} |\tilde{J}_{2\ell-1}(x) - \tilde{J}_{2\ell}
(x)|dx.
\end{eqnarray*}

Claim~\ref{claim:1} now follows immediately from
\[
\int_{2\ell-2}^{2\ell} |\tilde{J}_{2\ell-1}(x) - \tilde{J}_{2\ell}(x)|dx 
= \Omega(1),
\]
which is an easy consequence of Fact~\ref{fact:fourth-poly}.
\ignore{
}
\end{proof}

\begin{claim} \label{claim:2}
We have $\int_{0}^{2k} \sqrt{p_b(x) p_{b'}(x)} dx \geq 1-
O(\eps^2/k)$, 
i.e. $h(p_b,p_{b'})^2 \leq  O (\eps^2/k).$
\end{claim}

\begin{proof}
As above $\omega_{b,j}=\omega_{b',j}$
for all but exactly two (adjacent) values of $j$ which are $j=2\ell-1$ 
and $j=2\ell$.
For those values we have
\[
\omega_{b,2\ell-1} =   (1+\eps)/(2k), \quad \omega_{b',2\ell-1} 
= (1-\eps)/(2k), \quad
\omega_{b,2\ell} = (1-\eps)/(2k), \quad \omega_{b',2\ell} = (1+\eps)/(2k).
\]

We have

\[
h(p_b,p_{b'})^2
= {\frac 1 2} \int_{0}^{2k} \left(\sqrt{p_b}(x) - \sqrt{p_{b'}(x)}
\right)^2 dx
= A/2 + B/2,
\] 
where

\[
A = \int_{[2k] \setminus [2\ell-2,2\ell)}
              \left(\sqrt{p_b(x)} - \sqrt{p_{b'}(x)}\right)^2
dx
\]
and

\[
B = \int_{[2\ell-2,2\ell]}
 \left(\sqrt{p_b(x)} - \sqrt{p_{b'}(x)}\right)^2 dx.
\]

We first bound $B$, by upper bounding the value of the integrand
$\left(\sqrt{p_b(x)} - \sqrt{p_{b'}(x)}\right)^2$ on any fixed
$x \in [2k] \setminus [2\ell-2,2\ell].$
Recall that $p_b(x)$ is a sum of the $2k$ values 
$\omega_{b,j}\cdot \tilde{J}_j(x).$
The $0.1/(2k)$ contribution to each $\tilde{J}_j(x)$ ensures that
$p_b(x) \geq 0.1/(2k)$ for all $x \in [0,2k]$,
and it is easy
to see from the construction that $p_b(x) \leq 2/(2k)$ for all
$x \in [0,2k].$  The difference between the values $p_b(x)$
and $p_{b'}(x)$ comes entirely from $(\eps/k)
(\tilde{J}_{2\ell-1}(x) - \tilde{J}_{2\ell}(x))$, which has magnitude at most
$(\eps/k) \cdot (1/k^{4}) = \eps/k^{5}.$  So we have that
$\left(\sqrt{p_b}(x) - \sqrt{p_{b'}(x)}\right)^2$  is at most
the following (where $c_x \in [0.1,2]$ for each $x
\in [2k] \setminus [2\ell-2,2\ell]$):
\[
\left[
\sqrt{{\frac {c_x} k} + {\frac \eps {k^{5}}}} - \sqrt{{\frac {c_x} k}}
\right]^2
=
(c_x/k) \cdot \left[ \sqrt{1 + {\frac \eps {c_x k^{4}}}} - 1 \right]^2
=
(c_x/k) \cdot [\Theta(\eps/k^{4})]^2
= \Theta(\eps^2/k^{9}).
\]

Integrating over the region of width $2k-2$, we get that $B =O(\eps^2/k^{8}).$

It remains to bound $A$.  Fix any
$x \in [2\ell-2,2\ell]$.
As above we have that $p_b(x)$ equals $c_x/k$ for some $c_x \in [0.1,2]$,
and (\ref{eq:pb-alt}) implies that $p_b(x)$ and $p_{b'}(x)$
differ by at most $\Theta(\eps/k).$  So we have
\[
\left(\sqrt{p_b}(x) - \sqrt{p_{b'}(x)}\right)^2
\leq
\left[ \sqrt{{\frac {c_x} k}} - \sqrt{{\frac {c_x} k} - {\frac
{\Theta(\eps)} k}} \right]^2
= {\frac {c_x} k} \left[ 1 - \sqrt{1 - \Theta(\eps)} \right]^2
= {\frac {c_x} k} \Theta(\eps^2) = \Theta(\eps^2/k).
\]

Integrating over the region of width 2, we get that $A =
O(\eps^2/k).$  
\end{proof}

This concludes the proof of Theorem~\ref{thm:lower-bound-d-is-1}
and with it the proof of Theorem~\ref{thm:lower-bound-precise}.

}

\subsection{Proof of Lemma~\ref{lem:mix}.}
Recall Lemma~\ref{lem:mix}:

\medskip

\noindent {\bf Lemma~\ref{lem:mix}.}
\emph{
Let $p_1,\dots,p_k$ each be an
$(\tau,t)$-piecewise degree-$d$ distribution
over $[-1,1)$ and let $p = \sum_{j=1}^k \mu_j p_j$ be a $k$-mixture
of components $p_1,\dots,p_k.$  Then $p$ is a
$(\tau,kt)$-piecewise degree-$d$ distribution.
}

\medskip

\noindent {\bf Proof of Lemma~\ref{lem:mix}:}
For $1 \leq j \leq k$, let ${\cal P}_j$ denote the intervals
$I_{j,1},\dots,I_{j,t}$ such that $p_j$ is $\tau$-close to a
distribution $g_j$ whose pdf is given by polynomials $g_{j_1},\dots,
g_{j,t}$ over intervals $I_{j,1},\dots,I_{j,t}$ respectively.
Let ${\cal P}$ be the common refinement of ${\cal P}_1,\dots,{\cal P}_k.$
It is clear that ${\cal P}$ is a partition of $[-1,1)$ into at most $kt$
intervals.

For each $I$ in ${\cal P}$ and for each $1 \leq j \leq k$,
let $g_{j,I} \in \{g_{j,1},\dots,g_{j,t}\}$ be the polynomial
corresponding to $I$.
We claim that $p = \sum_{j=1}^k \mu_j p_j$ is $\tau$-close
to the $kt$-piecewise degree-$d$ distribution $g$
which has the polynomial $\sum_{j=1}^k \mu_j g_{j,I}$ as its pdf
over interval $I$, for each $I \in {\cal P}.$
To see this, for each interval $I \in {\cal P}$ let us write $\tilde{p}_{j,I}$
to denote the function which equals $p_j$ on $I$ and equals 0 elsewhere,
and likewise for $\tilde{g}_{j,I}.$
With this notation we may write the condition that $p_j$ is $\tau$-close
to $g_j$ in total variation distance as
\begin{equation} \label{eq:mix}
\left\|\sum_{I \in {\cal P}} \tilde{p}_{j,I} - \tilde{g}_{j,I}\right\|_1
\leq 2\tau.
\end{equation}
We then have
\begin{eqnarray*}
\|p - g\|_1 = \left \| \sum_{I \in {\cal P}}
\left(
\sum_{j=1}^k \mu_j \tilde{p}_{j,I} - \mu_j \tilde{g}_{j,I}
\right) \right\|_1 \leq
\sum_{j=1}^k \mu_j \left\| \sum_{I \in {\cal P}}
(\tilde{p}_{j,I} - \tilde{g}_{j,I}) \right\|_1 \leq 2\tau,
\end{eqnarray*}
and the proof is complete.
\qed

\subsection{Proof of Lemma~\ref{lem:FH}.}
Recall Lemma~\ref{lem:FH}:

\medskip

\noindent {\bf Lemma~\ref{lem:FH}.}
\emph{
With probability at least $99/100$, {\tt Find-Heavy}$(\gamma)$
returns a set $S$ satisfying conditions (1) and (2) in the ``Output''
description.
}

\medskip

\noindent {\bf Proof of Lemma~\ref{lem:FH}:}
Fix any $x \in [-1,1)$ such that
$\Pr_{x \sim p}[x] \geq 2 \gamma$.
A standard multiplicative Chernoff bound implies that $x$ is
placed in $S$ except with failure probability at most ${\frac 1 {200}}
\cdot {\frac 1 {2\gamma}}.$
Since there are at most ${\frac 1 {2 \gamma}}$ values $x \in [-1,1)$
such that $\Pr_{x \sim p}[x] \geq 2 \gamma$, we get that
condition (1) holds except with failure probability at most
${\frac 1 {200}}.$

For the second bullet, first consider any $x$ such that
$\Pr_{x \sim p}[x] \in [{\frac \gamma {2^c}}, {\frac \gamma 2}]$
(here $c>0$ is a universal constant).
A standard multiplicative Chernoff bound gives that each such $x$
satisfies $\widehat{p}(x) \geq 2 \Pr_{x \sim p}[x]$ with
probability at most ${\frac 1 {400}} \cdot {\frac {2^c} \gamma}$, and hence
each such $x$ satisfies $\widehat{p}(x) \geq \gamma$ with probability at most
${\frac 1 {400}} \cdot {\frac {2^c} \gamma}$. Since
there are at most $2^c/\gamma$ such $x$'s, we get
that with probability at least $1 - {\frac 1 {400}}$ no such $x$
belongs to $S$.

To finish the analysis we recall the following version of the
multiplicative Chernoff bound:

\begin{fact} \label{fact:MRCB} [\cite{MotwaniRaghavan:95}, Theorem~4.1]
Let $Y_1,\dots,Y_m$ be i.i.d. 0/1 random variables with $\Pr[Y_i=1]=q$
and let $Q=mq = \E[\sum_{i=1}^m Y_i].$
Then for all $\tau > 0$ we have
\[
\Pr\left[ \sum_{i=1}^m Y_i \geq (1 + \tau) Q
\right] \leq
\left({\frac {e^\tau}{(1+\tau)^{1+\tau}}}\right)^Q
\leq
\left({\frac {e}{(1+\tau)}}\right)^{(1+\tau)Q}.
\]
\end{fact}
Fix any integer $r \geq c$ and fix any $x$ such that
$\Pr_{x \sim p}[x] \in [{\frac \gamma {2^{r+1}}}, {\frac \gamma {2^r}}].$
Taking $1+\tau$ in Fact~\ref{fact:MRCB} to equal $2^r$, we get that
\[
\Pr[x \in S] \leq \left( {\frac e {2^r}} \right)^{\Theta(m \gamma)} =
\left( {\frac e {2^r}} \right)^{\Theta(\log(1/\gamma))}.
\]
Summing over all (at most $2^{r+1}/\gamma$ many) $x$ such that
$\Pr_{x \sim p}[x] \in [{\frac \gamma {2^{r+1}}}, {\frac \gamma {2^r}}]$,
we get that the probability that any such $x$ is placed in $S$ is at most
${\frac {2^{r+1}} \gamma} \cdot \left( {\frac e {2^r}} \right)^{
\Theta(\log(1/\gamma))} \leq {\frac 1 {400}} \cdot {\frac 1 {2^r}}.$
Summing over all $r \geq c$, the total failure probability incurred by such $x$
is at most $1/400.$  This proves the lemma.
\qed

\end{document}